\newcommand{\CLASSINPUTinnersidemargin}{0.5in} 
\newcommand{\CLASSINPUToutersidemargin}{0.5in} 
\newcommand{\CLASSINPUTtoptextmargin}{0.5in}   
\newcommand{\CLASSINPUTbottomtextmargin}{0.5in}

\documentclass[10pt,journal,compsoc]{IEEEtran}
\ifCLASSOPTIONcompsoc
  \usepackage[nocompress]{cite}
\else
  \usepackage{cite}
\fi

\ifCLASSINFOpdf
\else
 
\fi
\usepackage{times,graphicx,amsthm,comment,algorithm,algorithmic,amsmath,amsfonts,hyperref, mathtools, url, color, subfigure}
\usepackage{enumitem}

\DeclareMathOperator*{\Tr}{Tr}
\DeclareMathOperator*{\expec}{\mathbb E}
\newcommand{\MMD}{\text{MMD}}
\DeclareMathOperator*{\diag}{diag}
\DeclareMathOperator*{\argmax}{argmax}
\DeclareMathOperator*{\argmin}{argmin}
\DeclareMathOperator*{\cov}{Cov}

\DeclareMathOperator*{\hyp}{{Hyp}}
\DeclareMathOperator*{\disc}{\mathrm{disc}}

\DeclareMathOperator*{\discl}{\mathrm{disc}_{\ell}}

\DeclareMathOperator*{\radem}{\mathfrak{R}}

\renewcommand{\algorithmicrequire}{\textbf{Input:}}
\renewcommand{\algorithmicensure}{\textbf{Output:}}

\newcommand{\cF}{{\mathcal F}}
\newcommand{\cH}{{\mathcal H}}

\newcommand{\cL}{{\mathcal L}}

\newcommand{\cX}{{\mathcal X}}
\newcommand{\cY}{{\mathcal Y}}
\newcommand{\ba}{{\mathbf a}}
\newcommand{\bb}{{\mathbf b}}
\newcommand{\bB}{{\mathbf B}}

\newcommand{\bH}{{\mathbf H}}
\newcommand{\bI}{{\mathbf I}}
\newcommand{\bK}{{\mathbf K}}
\newcommand{\bL}{{\mathbf L}}
\newcommand{\bm}{{\mathbf m}}
\newcommand{\bM}{{\mathbf M}}

\newcommand{\bP}{{\mathbf P}}
\newcommand{\bQ}{{\mathbf Q}}
\newcommand{\bS}{{\mathbf S}}
\newcommand{\bt}{{\mathbf t}}

\newcommand{\bW}{{\mathbf W}}
\newcommand{\bx}{{\mathbf x}}
\newcommand{\bX}{{\mathbf X}}
\newcommand{\bz}{{\mathbf z}}
\newcommand{\bZ}{{\mathbf Z}}

\newcommand{\bmu}{\boldsymbol{\mu}}

\newcommand{\bbP}{{\mathbb P}}
\newcommand{\bbQ}{{\mathbb Q}}
\newcommand{\bbD}{{\mathbb D}}
\newcommand{\bbR}{{\mathbb R}}

\newtheorem{thm}{Theorem}
\newtheorem*{thm*}{Theorem}

\newtheorem{lem}[thm]{Lemma}
\newtheorem{defn}{Definition}

\newtheorem{rem}{Remark}

\begin{document}

\title{Scatter Component Analysis: A Unified Framework for Domain Adaptation and Domain Generalization}

\author{Muhammad~Ghifary,
	David Balduzzi, 
        W. Bastiaan~Kleijn,
        and~Mengjie~Zhang \\
Victoria University of Wellington \\
{\tt\small \{muhammad.ghifary,bastiaan.kleijn,mengjie.zhang\}@ecs.vuw.ac.nz, david.balduzzi@vuw.ac.nz}
}

\markboth{}%
{Ghifary \MakeLowercase{\textit{et al.}}}

\IEEEtitleabstractindextext{%

\begin{abstract}
This paper addresses classification tasks on a particular target domain in which labeled training data are only available from source domains different from (but related to) the target. Two closely related frameworks, domain adaptation and domain generalization, are concerned with such tasks, where the only difference between those frameworks is the availability of the unlabeled target data: domain adaptation can leverage unlabeled target information, while domain generalization cannot.
We propose \emph{Scatter Component Analyis} (SCA), a fast representation learning algorithm that can be applied to both domain adaptation and domain generalization. 
SCA is based on a simple geometrical measure, i.e., \emph{scatter}, which operates on \emph{reproducing kernel Hilbert space}.
SCA finds a representation that trades between maximizing the separability of classes, minimizing the mismatch between domains, and maximizing the separability of data; each of which is quantified through \emph{scatter}. The optimization problem of SCA can be reduced to a generalized eigenvalue problem, which results in a fast and exact solution. 
Comprehensive experiments on benchmark cross-domain object recognition datasets verify that SCA performs much faster than several state-of-the-art algorithms and also provides state-of-the-art classification accuracy in both domain adaptation and domain generalization.
We also show that \emph{scatter} can be used to establish a theoretical generalization bound in the case of domain adaptation.

\end{abstract}

}
\maketitle

\section{Introduction}
\label{sec:intro}
Supervised learning is perhaps the most popular task in machine learning and has recently achieved dramatic successes in many applications such as object recognition~\cite{Krizhevsky:2012aa,Simonyan2015}, object detection~\cite{girshick14CVPR}, speech recognition~\cite{dahl:13}, and machine translation~\cite{sutskever:14}.
These successes derive in large part from the availability of massive \emph{labeled} datasets such as PASCAL VOC2007 \cite{pascal-voc-2007} and ImageNet \cite{Krizhevsky:2009aa}.
Unfortunately, obtaining labels is often a time-consuming and costly process that requires human experts.
Furthermore, the process of collecting samples is prone to \emph{dataset bias}~\cite{Ponce2006,Torralba2011}, i.e., a learning algorithm trained on a particular dataset generalizes poorly across datasets.
In object recognition, for example, training images may be collected under specific conditions involving camera viewpoints, backgrounds, lighting conditions, and object transformations.
In such situations, the classifiers obtained with learning algorithms operating on samples from one dataset cannot be directly applied to other related datasets.
Developing learning algorithms that are robust to label scarcity and dataset bias is therefore an important and compelling problem.

\emph{Domain adaptation}~\cite{Blitzer:2006aa} and \emph{domain generalization}~\cite{Blanchard2011} have been proposed to overcome the fore-mentioned issues.  
In this context, a \emph{domain} represents a probability distribution from which the samples are drawn and is often equated with a dataset.
The domain is usually divided into two different types: the \emph{source domain} and the \emph{target domain}, to distinguish between a domain with labeled samples and a domain without labeled samples.
These two domains are related but different, which limits the applicability of standard supervised learning models on the target domain.
In particular, the basic assumption in standard supervised learning that training and test data come from the same distribution is violated.

The goal of domain adaptation is to produce good models on a target domain, by training on labels from the source domain(s) and leveraging \emph{unlabeled} samples from the target domain as supplementary information during training.
Domain adaptation has demonstrated significant successes in various applications, such as sentiment classification \cite{Chen:2012ab,Glorot:2011aa}, 
visual object recognition~\cite{Hoffman:2013aa,Long2014a,Saenko:2010aa,Shekhar:2013}, and WiFi localization~\cite{Pan:2009aa}.

Finally, the problem of domain generalization arises in situations where unlabeled target samples are not available, but samples from multiple source domains can be accessed.
Examples of domain generalization applications are automatic gating of flow cytometry~\cite{Blanchard2011,Muandet2013} and visual object recognition~\cite{Fang2013,Khosla2012,Xu2014}.

The main practical issue is that several state-of-the-art domain adaptation and domain generalization algorithms for object recognition result in optimization problems that are inefficient to solve~\cite{Long:2013aa,Long2014a,Shekhar:2013,Xu2014}.
Therefore, they may not be suitable in situations that require a real-time learning stage.
Furthermore, although domain adaptation and domain generalization are closely related problems, domain adaptation algorithms \emph{cannot} in general be applied directly to domain generalization, since they rely on the availability of (unlabeled) samples from the target domain.
It is highly desirable to develop algorithms that can be computed more efficiently, are compatible with both domain adaptation and domain generalization, and provides state-of-the-art performance.

\vspace{-1em}
\subsection{Goals and Objectives}
To address the fore-mentioned issues, we propose a fast unified algorithm for reducing dataset bias that can be used for both domain adaptation and domain generalization. 
The basic idea of our algorithm is to learn representations as inputs to a classifier that are invariant to dataset bias.
Intuitively, the learnt representations should incorporate four requirements: 
(i) separate points with different labels and 
(ii) separate the data as a whole (high variance), whilst 
(iii) not separating points sharing a label and 
(iv) reducing mismatch between the two or more domains.
The main contributions of this paper are as follows:
\begin{itemize}[leftmargin=*]
 \item The first contribution is \emph{scatter}, a simple geometric function that quantifies the mean squared distance of a distribution from its centroid. 
 We show that the above four requirements can be encoded through scatter and establish the relationship with Linear Discriminant Analysis \cite{Fisher1936}, Principal Component Analysis, Maximum Mean Discrepancy \cite{Borgwardt:2006aa} and Distributional Variance \cite{Muandet2013}.
 \item The second contribution is a fast scatter-based feature learning algorithm that can be applied to both domain adaptation and domain generalization problems, \emph{Scatter Component Analysis} (SCA), see Algorithm~1. To the best of our knowledge, SCA is the first multi-purpose algorithm applicable across a range of domain adaptation and generalization tasks. 
 The SCA optimization reduces to a generalized eigenproblem that admits a fast and exact solution on par with Kernel PCA~\cite{Scholkopf1998} in terms of time complexity.
 \item The third contribution is the derivation of a theoretical bound for SCA in the case of domain adaptation.
 Our theoretical analysis shows that \emph{domain scatter} 
 controls the generalization performance of SCA.
 We demonstrate that \emph{domain scatter} controls the \emph{discrepancy distance} under certain conditions. The discrepancy distance has previously been shown to control the generalization performance of domain adaptation algorithms~\cite{Mansour2009}. 
\end{itemize}

We performed extensive experiments to evaluate the performance of SCA against a large suite of alternatives in both domain adaptation and domain generalization settings.
We found that SCA performs considerably faster than the prior state-of-the-art across a range of visual object cross-domain recognition, with competitive or better performance in terms of accuracy.

\vspace{-1em}
\subsection{Organization of the Paper}
This paper is organized as follows.
Section~\ref{sec:litrev} describes the problem definitions and reviews existing work on domain adaptation and domain generalization.
Sections \ref{sec:scatter} and \ref{sec:sca} describes our proposed tool and also the corresponding feature learning algorithm, \emph{Scatter Component Analysis} (SCA).
The theoretical domain adaptation bound for SCA is then presented in Section \ref{sec:bound}.
Comprehensive evaluation results and analyses are provided in Sections~\ref{sec:exp1} and \ref{sec:exp2}.
Finally, Section \ref{sec:conc} concludes the paper.

\vspace{-1em}
\section{Background and Literature Review}
\label{sec:litrev}
This section establishes the basic definitions of domains, domain adaptation, and domain generalization.
It then reviews existing work in domain adaptation and domain generalization, particularly in the area of computer vision and object recognition.

A \emph{domain} is a probability distribution $\bbP_{XY}$ on $\cX \times \cY$, where $\cX$ and $\cY$ are the input and label spaces respectively. 
For the sake of simplicity, we equate $\bbP_{XY}$ with $\bbP$.
The terms domain and distribution are used interchangeably throughout the paper. 
Let $S = \{ x_i, y_i\}_{i=1}^{n} \sim \bbP$ be an i.i.d. sample from a domain. It is convenient to use the notation $\hat{\bbP}$ for the corresponding empirical distribution $\hat{\bbP}(x,y)=\frac{1}{n}\sum_{i=1}^n \delta_{(x_i,y_i)}(x,y)$, where $\delta$ is the Dirac delta.
We define \emph{domain adaptation} and \emph{domain generalization} as follows.
\begin{defn}[\textbf{Domain Adaptation}]
  \label{def:domadap}
 Let $\bbP^s$ and $\bbP^t$ be a source and target domain respectively, where $\bbP^s \neq \bbP^t$.
 Denote by $S^s = \{x^s_i, y^s_i \}_{i=1}^{n_s} \sim \bbP^s$ and $S^t_u = \{x^t_i\}_{i=1}^{n_t} \sim \bbP^t_X$ samples drawn from both domains.
 The task of domain adaptation is to learn a good labeling function $f_{\bbP^t} : \cX \rightarrow \cY$ given $S^s$ and $S^t_u$ as the training examples.
\end{defn}

\begin{defn}[\textbf{Domain Generalization}]
  \label{def:domgen}
  Let $\Delta = \{ \bbP^1,\ldots, \bbP^m \}$ be a set of $m$ source domains and $\bbP^t \notin \Delta$ be a target domain.
  Denote by $S^d = \{ x^d_i, y^d_i\}_{i=1}^{n_d} \sim \bbP^d$ samples drawn from $m$ source domains.
  The task of domain generalization is to learn a labeling function $f_{\bbP^t}: \cX \rightarrow \cY$ given $S^d, \forall d=1,...,m$ as the training examples.
\end{defn}
It is instructive to compare these two related definitions.
The main difference between domain adaptation and domain generalization is on \emph{the availability of the unlabeled target samples}.
Both have the same goal: learning a labeling function $f: \cX \rightarrow \cY$ that performs well on the target domain.
In practice, domain generalization requires $m>1$ to work well although $m=1$ might not violate Definition~\ref{def:domgen}.
Note that domain generalization can be exactly reduced to domain adaptation if $m=2$ and $\bbP^t_X \in \Delta$.

Domain adaptation and domain generalization have recently attracted great interest in machine learning.
We present a review of recent literature that is organized into two parts: i) domain adaptation and ii) domain generalization.

\vspace{-1em}
\subsection{Domain Adaptation}
Earlier studies on domain adaptation focused on natural language processing, see, e.g., \cite{Jiang:2008aa} and references therein.
Domain adaptation has gained increasing attention in computer vision for solving dataset bias in object recognition \cite{Saenko:2010aa,Gong:2012aa,Tommasi:2013ab,Fernando:2013aa,Hoffman_CVPR2014,Long2014a} and object detection \cite{Sun:BMVC2014}.
The reader is encouraged to consult the recent survey in visual domain adaptation~\cite{patel_dasurvey:2015} for a more comprehensive review.
We classify domain adaptation algorithms into three categories: i) the classifier adaptation approach, ii) the selection/reweighting approach, and iii) the feature transformation-based approach.

The \emph{classifier adaptation approach} aims to learn a good, adaptive classifier on a target domain by leveraging knowledge from source or auxiliary domains \cite{Yang:2007aa,DAM:2012,DuanTPAMI2012a,DuanTPAMI2012b,Niu_IJCV16}.
Adaptive Support Vector Machines (A-SVMs)~\cite{Yang:2007aa} utilize \emph{auxiliary classifiers} to adapt a \emph{primary classifier} that performs well on a target domain, where the optimization criterion is similar to standard SVMs. 
The Domain Adaptation Machine (DAM)~\cite{DAM:2012} employs both a domain-dependent regularizer based on a smoothness assumption and a sparsity regularizer in Least-Squares SVMs~\cite{LSSVM:2004}.
Recently, a multi-instance learning based classifier for action and event recognition, trained on weakly labeled web data, was proposed \cite{Niu_IJCV16}.

The \emph{reweighting/selection approach} reduces sample bias by reweighting or selecting source instances that are `close' to target instances -- selection can be considered as the `hard' version of reweighting.
The basic idea has been studied under the name of \emph{covariate shift}~\cite{Shimodaira:2000aa}.
Gong et al.\cite{Gong:2013ab} applied a convex optimization strategy to select some source images that are maximally similar to the target images according to Maximum Mean Discrepancy~\cite{Borgwardt:2006aa} -- referred to as \emph{landmarks}.
The landmarks are then used to construct multiple auxiliary tasks as a basis for composing domain-invariant features.
Transfer Joint Matching (TJM)~\cite{Long2014a} uses a reweighting strategy as a regularizer based on $\ell_{2,1}$-norm structured sparsity on the source subspace bases.

The \emph{feature transformation-based approach} is perhaps the most popular approach in domain adaptation.
Daume III \cite{Daume-III:2007aa} proposed a simple feature augmentation method by replicating the source and target data, both are in $\mathbb{R}^d$, as well as zero-padding such that the resulting features are in $\mathbb{R}^{3d}$. 
Li et al. \cite{Li_HFA:TPAMI2014} extended the method for the case of heterogeneous features, i.e., source and target features have different dimensionality, by introducing a common subspace learnt via the standard SVM formulation.
A subspace learning-based algorithm, Transfer Component Analysis (TCA) and its semi-supervised version SSTCA~\cite{Pan2011}, utilizes the Maximum Mean Discrepancy (MMD)~\cite{Gretton:2007aa} to minimize dataset bias in WiFi localization and text classification applications.
Metric learning-based domain adaptation approaches have been proposed \cite{Saenko:2010aa, Kulis:2011aa}, which were early studies in object recognition on the Office dataset.
The idea of extracting `intermediate features' to minimize dataset bias by projecting data onto multiple intermediate subspaces was also considered.
Sampling Geodesic Flow (SGF) \cite{Gopalan:2011aa} and Geodesic Flow Kernel (GFK) \cite{Gong:2012aa} generate multiple subspaces via an interpolation between the source and the target subspace on a Grassmann manifold -- a point on the manifold is a subspace.
Subspace Alignment (SA) \cite{Fernando:2013aa} transforms a source PCA subspace into a new subspace that is well-aligned to a target PCA subspace without requiring intermediate subspaces.
A recent method called Correlation Alignment (CORAL) facilitates adaptive features by aligning the source and target covariance matrices \cite{Sun:AAAI2016}.
Other subspace learning-based methods such as Transfer Sparse Coding (TSC) \cite{Long:2013aa} and Domain Invariant Projection (DIP) \cite{Baktashmotlagh:2013} make use of MMD, following TCA, to match the source and target distributions in the feature space.
One of the methods proposed in \cite{Baktashmotlagh:2014} follows a similar intuition by using Hellinger distance as an alternative to MMD.
Algorithms based on hierarchical non-linear features or deep learning are also capable of producing powerful domain adaptive features 
\cite{Chopra:2013aa,Donahue:2014aa,Ganin2015,Ghifary2014b,Hoffman:2013aa,Long_DAN:2015}.

Several works have addressed Probably Approximately Correct theoretical bounds for domain adaptation.
Ben-David et al.~\cite{Ben-David:2007aa} presented the first theoretical analysis of domain adaptation, that is, an adaptation bound in classification tasks based on the $d_{\mathcal{A}}$-distance \cite{Kifer:2004}.
Mansour et al.~\cite{Mansour2009} extended this work in several ways built on Rademacher complexity~\cite{Bartlett:2002} and the \emph{discrepancy distance}, as an alternative to $d_{\mathcal{A}}$-distance.
In this paper, we provide a domain adaptation bound for our new algorithm based on the latter analysis.

\vspace{-0.5em}
\subsection{Domain Generalization}
Domain generalization is a newer line of research than domain adaptation.
Blanchard et al.~\cite{Blanchard2011} first studied this issue and proposed an augmented SVM that encodes empirical marginal distributions into the kernel for solving automatic gating of flow cytometry.
A feature projection-based algorithm, Domain-Invariant Component Analysis (DICA)~\cite{Muandet2013}, was then introduced to solve the same problem.
DICA extends Kernel PCA~\cite{Scholkopf1998} by incorporating the \emph{distributional variance} to reduce the dissimilarity across domains and the central subspace~\cite{COIR:2011} to capture the functional relationship between the features and their corresponding labels.

Domain generalization algorithms also have been used in object recognition.
Khosla et al.~\cite{Khosla2012} proposed a multi-task max-margin classifier, which we refer to as Undo-Bias, that explicitly encodes dataset-specific biases in feature space.
These biases are used to push the dataset-specific weights to be similar to the global weights. 
Fang et al.\cite{Fang2013} developed Unbiased Metric Learning (UML) based on a learning-to-rank framework. 
Validated on weakly-labeled web images, UML produces a less biased distance metric that provides good object recognition performance.
Xu et al.\cite{Xu2014} extended an exemplar-SVM~\cite{Malisiewicz:2011} to domain generalization by adding a nuclear norm-based regularizer that captures the likelihoods of all positive samples. 
The proposed model is referred to as LRE-SVM that provides the state-of-the-art performance.
More recently, an autoencoder based algorithm to extract domain-invariant features via multi-task learning has been proposed \cite{Ghifary:ICCV2015}.

Although both domain adaptation and domain generalization have the same goal (reducing dataset bias), the approaches are generally not compatible to each other -- domain adaptation methods cannot be directly applied to domain generalization or vice versa.
To our best knowledge, only LRE-SVM can be applied to both domain adaptation and domain generalization.
The domain generalization algorithm formulation as in DICA, Undo-Bias, or UML typically does not allow to take into account unlabeled data from the target domain.
Furthermore, several state-of-the-art domain adaptation and domain generalization algorithms such as TJM and LRE-SVM, require the solution of a computationally complex optimization that induces high complexity in time.
In this work, we establish a fast algorithm that overcomes the above issues.

\section{Scatter}
\label{sec:scatter}
We work in a feature space that is a reproducing kernel Hilbert space (RKHS) $\cH$.
The main motivation is to transform original inputs onto $\cH$, which is high or possibly infinite dimensional space, with the hope that the new features are linearly separable.
The most important property of RKHS is perhaps to allow a computationally feasible transformation onto $\cH$ by virtue of the \emph{kernel trick}.
\begin{defn}[\textbf{Reproducing Kernel Hilbert Space}] 
\label{d:rkhs}
 Let $\cX$ be an arbitrary set and $\cH\subset\{f:\cX\rightarrow\bbR\}$ a Hilbert space of functions on $\cX$. Define the evaluation functional $L_x: \cH \rightarrow \bbR$ by $L_x[h] := h(x), \forall h \in \cH$. Then $\cH$ is a \textbf{reproducing kernel Hilbert space (RKHS)} if the functional $L_x$ is always bounded: i.e.  for all $x\in\cX$ there exists an $\lambda > 0$ such that
 \begin{equation}
 \label{eq:rkhs}
  | L_x[h] | = | h(x) | \leq \lambda \| h \|_{\cH}.
 \end{equation}
 It follows that there is a function $\phi : \cX \rightarrow \cH$ (referred to as the \textbf{canonical feature map}) satisfying: 
  \begin{equation}
  \label{eq:repprop}
  L_x[h] = \langle h, \phi(x) \rangle = h(x)
  \quad\text{for all $h \in \cH$ and $x\in\cX$}.
  \end{equation}
  Consequently, for each $t \in \cX$, one can write
  \begin{equation}
   \langle \phi(t), \phi(x) \rangle =: \kappa(t,x). \nonumber
  \end{equation}
  The function $\kappa : \cX \times \cX \rightarrow \bbR$ is referred to as the \textbf{reproducing kernel}.
\end{defn}
\noindent Expression (\ref{eq:rkhs}) is the weakest condition that ensures the existence of an inner product and also the ability to evaluate each function in $\cH$ at every point in the domain, while (\ref{eq:repprop}) provides more useful notion in practice.

Before introducing scatter, it is convenient to first represent domains as points in RKHS using the mean map \cite{smola:07}:
\begin{defn}[\textbf{Mean map}]
  \label{d:meanmap}
	Suppose that $\cX$ is equipped with a kernel, and that $\cH$ is the corresponding RKHS with feature map $\phi:\cX\rightarrow \cH$. Let $\Delta_\cX$ denote the set of probability distributions on $\cX$. The mean map takes distributions on $\cX$ to points in $\cH$:
	\begin{equation*}
		\label{eq:mean-map}
		\mu:\Delta_\cX\rightarrow \cH:\bbP\mapsto \expec_{x\sim \bbP}\big[ \phi(x)\big] =: \mu_\bbP.
	\end{equation*}
\end{defn}
\noindent Geometrically, the mean map is the centroid of the image of the distribution under $\phi$.
We define \emph{scatter} as the variance of points in the image around its centroid:

\begin{defn}[\textbf{Scatter}]
	\label{d:scatter}
	The \textbf{scatter} of distribution $\bbP$ on $\cX$ relative to $\phi$ is
	\begin{equation*}
		\label{e:scatter}
		\Psi_\phi(\bbP) := \expec_{x\sim \bbP} \Big[\big\|\mu_\bbP - \phi(x)\big\|^2_\cH\Big]
	\end{equation*}
	where $\|\cdot\|_\cH$ is the norm on $\cH$.
\end{defn}

The scatter of a domain cannot be computed directly; instead it is estimated from observations. 
The scatter of a finite set of observations $\{x_1,\ldots, x_n\}$ is computed with respect to the empirical distribution 
\begin{equation*}
	\widehat{\bbP}(x):=\frac{1}{n}\sum_{i=1}^n \delta_{x_i}(x)
	\quad\text{where}\quad
	\delta_{x_i}(x)=\begin{cases}
		1 & \text{if }x_i=x\\
		0 & \text{else.}
	\end{cases}
\end{equation*}

We provide a theorem that shows how the difference between the true scatter and a finite sample estimate decreases with the sample size.
\begin{thm}[\textbf{Scatter Bound}] 
	\label{t:bound}
	Suppose $\bbP$ is a true distribution over all samples of size $n$ and $\hat{\bbP}$ is its empirical distribution. Further suppose that $\|\phi(x)\|^2 \leq M$ for all $x\in\cX$. Then, with probability $\geq 1-\delta$,
	\begin{equation*}
		\big|\Psi_\phi(\bbP) - \Psi_\phi(\hat{\bbP})\big| 
		\leq M\sqrt{\frac{2 \log(\frac{2}{\delta})}{n}}.
	\end{equation*}
\end{thm}

\begin{proof} See Supplementary material. \end{proof}
\noindent Note that the right hand site of the bound is smaller for lower values of $M$ and higher values of $n$. 
Furthermore, if $\kappa$ is in the form of Gaussian kernel, the bound only depends on $n$ since $M=1$.

We provide an example for later use. If the input space is a vector space and $\phi$ is the identity then it follows immediately that
\begin{lem}[\textbf{Total variance as scatter}]	
	\label{t:total_variance}
	The scatter of the set of $p$-dimensional points (in a matrix form) $\bX=[\bx_1,\ldots,\bx_n]^{\top} \in \bbR^{n \times p}$ relative to the identity map 
	$\phi(\bx) := \bx$, is the total variance:
	\begin{equation*}
		\Psi(\bX) = \frac{1}{n}\Tr (\bX - \bar{\bX})^\top (\bX - \bar{\bX}) = \Tr \cov(\bX),
	\end{equation*}
	where $\Tr(\cdot)$ denotes the trace operation and $\bar{\bX} = [\bar{\bx}, \ldots, \bar{\bx}]^\top$ with $\bar{\bx} = \sum_{i=1}^n \bx_i$.
\end{lem}

We utilize \emph{scatter} to formulate a feature learning algorithm referred to as Scatter Component Analysis (SCA).
Specifically, \emph{scatter} quantifies requirements needed in SCA to develop an effective solution for both domain adaptation and generalization, which will be described in the next section.

\vspace{-1em}
\section{Scatter Component Analysis (SCA)}
\label{sec:sca}
SCA aims to efficiently learn a representation that improves both domain adaptation and domain generalization. 
The strategy is to convert the observations into a configuration of points in feature space such that the domain mismatch is reduced. 
SCA then finds a representation of the problem (that is, a linear transformation of feature space) for which 
(i) the source and target domains are similar and 
(ii) elements with the same label are similar; whereas 
(iii) elements with different labels are well separated and 
(iv) the variance of the whole data is maximized.
Each requirement can be quantified through \emph{scatter} that leads to four consequences: (i) \emph{domain scatter}, (ii) \emph{between-class scatter}, (iii) \emph{within-class scatter}, and (iv) \emph{total scatter}.

The remainder of the subsection defines the above four scatter quantities in more detail (along the way relating the terms to principal component analysis, the maximum mean discrepancy, and Fisher's linear discriminant) and describes the SCA's learning algorithm.
We will also see that SCA can be easily switched to either domain adaptation or domain generalization by modifying the configuration of the input domains.

\subsection{Total Scatter}
Given $m$ domains $\bbP^1_X,\ldots,\bbP^m_X$ on $\cX$, we define the total domain as the mean $\bar{\bbP}_{X} = \frac{1}{m} \sum_{d=1}^{m}\bbP^d_X$.
The \emph{total scatter} is then defined by
\begin{equation}
    \text{total scatter } = \Psi_\phi\left(\bar{\bbP}_X \right).
\end{equation}
It is worth emphasizing that this definition is general in the sense that it covers both domain adaptation ($m = 2$ and one of them is the target domain) and domain generalization ($m>2$).

Total scatter is estimated from data as follows. 
Let $\bX = [\bx_1, ..., \bx_n]^{\top} \in \bbR^{n \times p}$ be the matrix of unlabeled samples from all $m$ domains ($n = \sum_{d=1}^m {n_d}$, where $n_d$ is the number of examples in the $d$-th domain).
Given a feature map $\phi: \bbR^{p} \rightarrow \cH$ corresponding to kernel $\kappa$, define a set of functions arranged in a column vector $\mathbf{\Phi} = [\phi(\bx_1), ..., \phi(\bx_n)]^{\top}$.
After centering $\{ \phi(\bx_i)\}_{i=1}^{n}$ by subtracting the mean,
the covariance matrix is $\cov(\mathbf{\Phi}) =  \frac{1}{n}\mathbf{\Phi}^{\top} \mathbf{\Phi}$. By Lemma~\ref{t:total_variance}, 
\begin{equation}
 	\Psi_\phi\left(\hat{\bar{\bbP}}_X \right)
 	= \Tr\cov(\mathbf{\Phi}).
\end{equation}

We are interested in the total scatter after applying a linear transform to a finite relevant subspace $\bW:\cH\rightarrow \bbR^k$.
To avoid the direct computation of $\phi: \cX \rightarrow \cH$, which could be expensive or undoable, we use the \emph{kernel trick}.
Let $\bZ = \mathbf{\Phi} \bW \in \mathbb{R}^{n \times k}$ be the $n$ transformed feature vectors and
$[\bK]_{ij} = [\mathbf{\Phi} \mathbf{\Phi}^{\top}]_{ij} = [\kappa(\bx_i, \bx_j)]$.
After fixing $\bB \in \bbR^{n \times k}$ such that $\bW = \mathbf{\Phi}^{\top} \bB$, 
the total transformed scatter is 
\begin{eqnarray}
	\label{eq:total_scatter}
	\Psi_{\bB\circ \phi}\big(\hat{\bar{\bbP}}_X \big)
	= \Tr{(\frac{1}{n}\bB^{\top} \bK \bK \bB)}.
\end{eqnarray}
We remark that, in our notation, Kernel Principal Component Analysis (KPCA)~\cite{Scholkopf1998} corresponds to the optimization problem
\begin{eqnarray}
	\label{eq:kpca_opt2}
	\max  \Psi_{\bB\circ \phi}\big(\hat{\bar{\bbP}}_X \big).
\end{eqnarray}

\subsection{Domain Scatter}
Suppose we are given $m$ domains $\bbP^1_X,\ldots, \bbP^m_X$ on $\cX$. 
We can think of the \emph{set} $\{\mu_{\bbP^1_X},\ldots,\mu_{\bbP^m_X}\}\subset \cH$ as a sample from some latent distribution on domains. Equipping the sample with the empirical distribution and computing scatter relative to the identity map on $\cH$ yields \emph{domain scatter}:
\begin{equation}
	\Psi\big(\{\mu_{\bbP^1_X},\ldots,\mu_{\bbP^m_X}\}\big) = \frac{1}{m}\sum_{i=1}^m\big\|\bar{\mu} - \mu_{\bbP^i}\big\|^2,
\end{equation}
where $\bar{\mu} = \frac{1}{m}\sum_{i=1}^m \mu_{\bbP^i}$.
Note that domain scatter coincides with the \emph{distributional variance} introduced in \cite{Muandet2013}.
Domain scatter is also closely related to the Maximum Mean Discrepancy (MMD), used in some domain adaptation algorithms~\cite{Huang:2007,Pan2011,Long2014a}.

\begin{defn}
	Let $\cF$ be a set of functions $f:\cX\rightarrow \bbR$. The \textbf{maximum mean discrepancy} between domains $\bbP$ and $\bbQ$ is
	\begin{equation*}
		\label{eq:mmd}
		\MMD_\cF[\bbP, \bbQ] := \sup_{f\in\cF} \left(\expec_{\bbP}\left[f(x)\right] - \expec_{\bbQ}\left[f(x)\right]\right).
	\end{equation*}
\end{defn}
The MMD measures the extent to which two domains resemble one another from the perspective of function class $\cF$. 
The following theorem relates domain scatter to MMD given two domains, where the case of interest is bounded linear functions on the feature space:

\begin{lem}[\textbf{Scatter recovers MMD}]
	\label{t:mmd}
	The scatter of domains $\bbP$ and $\bbQ$ on $\cX$ is their (squared) maximum mean discrepancy:
	\begin{equation*}
		\Psi(\{\mu_{\bbP},\mu_{\bbQ}\}) = \frac{1}{4}\MMD^2_\cF[\bbP,\bbQ],
	\end{equation*}
	where $\cF=\{f:\cX\rightarrow \bbR \,|\, f\text{ is linear and }\|f\|_\cF\leq 1\}$.

	In particular, if $\phi$ is induced by a characteristic kernel on $\cX$ then $\Psi(\{\mu_{\bbP},\mu_{\bbQ}\}) = 0$ if and only if $\bbP=\bbQ$.
\end{lem}

\begin{proof}
	Note that the theorem involves two levels of probability distributions: 
	(i) the domains $\bbP$ and $\bbQ$ on $\cX$, and 
	(ii) the empirical distribution on $\cF$ that assigns probability $p=\frac{1}{2}$ to the points $\mu_\bbP$ and $\mu_\bbQ$, and $p=0$ to everything else.
	Let $\bar{\mu}=\frac{1}{2}(\mu_\bbP + \mu_\bbQ)$. By Definition~\ref{d:scatter},
	\begin{equation*}
		\Psi(\{\mu_\bbP,\mu_\bbQ\}) = \frac{1}{2}\|\bar{\mu} - \mu_\bbP\|^2_\cF + \frac{1}{2}\|\bar{\mu} - \mu_\bbQ\|^2_\cF = \frac{1}{4}\|\mu_\bbP-\mu_\bbQ\|^2_\cF.
	\end{equation*}
	The result follows from Theorem~2.2 of \cite{Borgwardt:2006aa}.
\end{proof}
Lemma~\ref{t:mmd} also tells that the domain scatter is a valid metric if the kernel on $\cX$ is characteristic~\cite{sriperumbudur:10}. The most important example of a characteristic kernel is the Gaussian RBF kernel, which is the kernel used in the theoretical results and experiments below.
We also remark that MMD can be estimated from observed data with bound provided in \cite{Gretton:2012aa}, which is analogous to Theorem~\ref{t:bound}.

Domain scatter in a transformed feature space in $\bbR^k$ is estimated as follows. Suppose we have $m$ samples $S^d_u = \{ \mathbf{x}^d_i \}_{i=1}^{n_d} \sim \bbP^d_X$.
Recall that $\bZ =  \mathbf{\Phi} \bW = \bK^{\top} \bB$, where $\bZ = [\bz_1, \ldots, \bz_n]^{\top}$ contains projected samples from all domains: $\mathbf{z}_i = \mathbf{W}^{\top} \phi(\mathbf{x}_i)$ and
\begin{equation}
\label{eq:kernel}
\bK = 
	\begin{bmatrix} 
	  \bK^{11} & \cdots & \bK^{1m}\\
	  \vdots & \ddots & \vdots \\
	  \bK^{m1} & \cdots & \bK^{mm}
	\end{bmatrix} \in \bbR^{n \times n}
\end{equation}
is the corresponding kernel matrix, where $[\bK^{kl}]_{ij} = \kappa(\mathbf{x}^k_i, \mathbf{x}^l_j)$.
By some algebra, the domain scatter is
\begin{eqnarray}
	\label{eq:domain_scatter}
	\Psi_{\bB}\big( \{ \mu_{\hat{\bbP}^d_X}\}_{d=1}^m\big) 
	= \Tr(\bB^{\top} \bK \bL \bK \bB),
\end{eqnarray}
where $\bL \in \bbR^{n \times n}$ is a coefficient matrix
with $[\bL^{kl}]_{ij} = \frac{m-1}{m^2 n^2_k}$ if $k = l$, and $-\frac{1}{m^2 n_k n_l}$ otherwise.

\subsection{Class Scatter}
For each class $k\in\{1,\ldots,C\}$, let $\bbP_{X|k}^l$ denote the conditional distribution on $\cX$ induced by the total labeled domain $\bbP_{XY}^l = \frac{1}{q} \sum_{j=1}^{q}\bbP^j_{XY}$ when $Y=k$  (the number of labeled domains $q$ does not necessarily equal to the number of source domains $m$). 
We define the \emph{within-class scatter} and \emph{between-class scatter} as
\begin{equation}
  \underbrace{\Psi_\phi(\bbP_{X|k}^l)}_{\text{within-class-$k$ scatter}}
  \quad\text{ and }\quad
  \underbrace{\Psi\big(\{\mu_{\bbP_{X|k=1}^l},\ldots, \mu_{\bbP_{X|k=C}^l}\}\big)}_{\text{between-class scatter}}.
\end{equation}

The class scatters are estimated as follows. Let $\bS_k^w = \big(\phi(\bx_j)\big)_{\bx_j\in k}$ denote the $n_k$-tuple of source samples in class $k$. The centroid of $\bS_k^w$ is $\bmu_k= \frac{1}{n_k}\sum_{\bx_i\in k} \phi(\bx_i)$. Furthermore, let $\bS^b = (\bmu_1,\ldots,\bmu_{|C|})$ denote the $n$-tuple of all class centroids \emph{where centroid $k$ appears $n_k$ times in $\bS^b$}. The centroid of $\bS^b$ is then the centroid of the source domain: $\bar{\bmu}^s = \frac{1}{n}\sum_{k=1}^{|C|}n_k \bmu_k$. It follows that the within-class scatter is
\begin{equation*}
	\Psi_{\phi}\Big(\hat{\bbP}^l_{X|y_k}\Big) = \Tr\left(\sum_{j=1}^{n_k} \left( \phi(\bx_{jk}) - \bmu_k \right) \left( \phi(\bx_{jk}) - \bmu_k \right)^{\top}\right)
\end{equation*}
 and the between-class scatter is
\begin{equation*}
	\Psi\Big(\big\{\mu_{\hat{\bbP}^l_{X|y_k}}\big\}_{k=1}^C\Big) = \Tr\left( n_k (\bmu_k - \bar{\bmu}) (\bmu_k - \bar{\bmu})^{\top}\right).
\end{equation*}
The right-hand sides of the above equations are the classical definitions of within- and between- class scatter \cite{Fisher1936}. The classical linear discriminant is thus a ratio of scatters
\begin{equation*}
	\text{Fisher's linear discriminant}= 
	  \frac{\Psi\Big(\big\{\mu_{\hat{\bbP}^l_{X|y_k}}\big\}_{k=1}^C\Big)}
	      {\sum_{k=1}^C \Psi_{\phi}\Big(\hat{\bbP}^l_{X|y_k}\Big)}.
\end{equation*}
Maximizing Fisher's linear discriminant increases the separation of the data points with respect to the class clusters.

Given a linear transformation $\bW : \cH \rightarrow \bbR^k$, it follows from Lemma~\ref{t:total_variance} that the class scatters in the projected feature space $\tilde{\cH}$ are 
\begin{eqnarray}	 
	\label{eq:bclass_scatter}
	\Psi_{\bB}\Big(\big\{\mu_{\hat{\bbP}^l_{X|y_k}}\big\}_{k=1}^C\Big)
	 &=& \Tr (\bW^{\top} \cov(\bS^b)  \bW) \nonumber \\ &=& \Tr (\bB^{\top} \bP_s \bB),\\
	 \label{eq:wclass_scatter}
	\sum_{k=1}^{C} \Psi_{\bB\circ \phi}\Big(\hat{\bbP}^s_{X|y_k}\Big) 
	&=& \sum_{k=1}^{C} \Tr ( \bW^{\top} \cov(\bS^w_k) \bW) \nonumber \\ &=&  \Tr(\bB^{\top} \bQ_s \bB),
\end{eqnarray}
where
\begin{eqnarray}
	\label{eq:between_class_kernel}
	\bP_s &=& \sum_{k=1}^{C} n_k (\bm_k - \bar{\bm})(\bm_k - \bar{\bm})^{\top}, \\
	\label{eq:within_class_kernel}
	\bQ_s  &=& \sum_{k=1}^{C} \bK_k \bH_k \bK^{\top}_k,
\end{eqnarray}
with 
$\bm_k = \frac{1}{n_k} \sum_{j=1}^{n_k} \kappa (\cdot, \bx_{jk})$, 
$\bar{\bm} = \frac{1}{n} \sum_{j=1}^{n} \kappa (\cdot, \bx_j)$,
$[\bK_k]_{ij} = [\kappa(\bx_{ik}, \bx_{jk})]$, and 
the centering matrix $\bH_k = \bI_{n_k} - \frac{1}{n_k}\mathbf{1}_{n_k} \mathbf{1}_{n_k}^{\top}$, where
$\bI_{n_k}$ denotes a $n_k \times n_k$ identity matrix and $\mathbf{1}_{n_k} \in \bbR^{n_k}$ denotes a vector of ones.

\subsection{The Algorithm}
Here we formulate the SCA's learning algorithm by incorporating the above four quantities.
The objective of SCA is to seek a representation by solving an optimization problem in the form of the following expression
\begin{equation}
	\label{e:idea}
	\sup \frac{\{\text{total scatter}\} + \{\text{between-class scatter}\}}{\{\text{domain scatter\}} + \{\text{within-class scatter}\}}.
\end{equation}

Using (\ref{eq:total_scatter}), (\ref{eq:domain_scatter}), (\ref{eq:bclass_scatter}), and (\ref{eq:wclass_scatter}),
the above expression can then be specified in more detail:
\begin{eqnarray}
	\label{eq:ftca_general}
	\argmax_\bB 
	\frac{\Psi_{\bB\circ\phi}\Big( \hat{\bar{\bbP}}_X\Big) 
	+ \Psi_{\bB}\Big(\big\{\mu_{\hat{\bbP}^l_{X|y_k}}\big\}_{k=1}^C\Big)}
	{\Psi_{\bB}\Big(\big\{ \mu_{\hat{\bbP}^d_X} \}_{d=1}^m \Big) 
	+ \sum_{k=1}^C\Psi_{\bB\circ\phi}\Big(\hat{\bbP}^l_{X|y_k}\Big)}.
\end{eqnarray}
Maximizing the numerator encourages SCA to preserve the total variability of the data and the separability of classes. Minimizing the denominator encourages SCA to find a representation for which the source and target domains are similar, and source samples sharing a label are similar.

\textbf{Objective function.}
We reformulate \eqref{eq:ftca_general} in three ways. 
First, we express it in terms of linear algebra. 
Second, we insert hyper-parameters that control the trade-off between scatters as one scatter quantity could be more important than others in a particular case.
Third, we impose the constraint that $\bW^\top\bW=\bB^\top\bK \bB$ is small to control the scale of the solution.

Explicitly, SCA finds a projection matrix $\bB = [\bb_1, \bb_2, ..., \bb_k] $ that solves the constrained optimization
\begin{equation}
	\label{eq:ftca_opt}
	\argmax_{\bB \in \bbR^{n \times k}} 
	\frac{\Tr\big( \bB^{\top} ( \frac{(1-\beta)}{n} \bK \bK+ \beta \bP) \bB \big)}
	{\Tr\big(\bB^{\top}  ( \delta \bK \bL \bK  +   \bQ + \bK) \bB \big)},
\end{equation}
where 
\begin{eqnarray*}
	\label{eq:big_pq}
	\bP = 
	\begin{bmatrix} 
		\bP_{s} & \mathbf{0}_{n_s \times n_t} \\ 
		\mathbf{0}_{n_t \times n_s} & \mathbf{0}_{n_t \times n_t}
	\end{bmatrix},
	\bQ =
	\begin{bmatrix}
		\bQ_{s} & \mathbf{0}_{n_s \times n_t} \\ 
		\mathbf{0}_{n_t \times n_s} & \mathbf{0}_{n_t \times n_t}
	\end{bmatrix},
\end{eqnarray*}
and $\beta, \delta > 0$ are the trade-off parameters controlling the total and between-class scatter, and domain scatter respectively.

Observe that the above optimization is invariant to rescaling $\bB \mapsto \alpha  \bB$.
Therefore, optimization (\ref{eq:ftca_opt}) can be rewritten as
\begin{eqnarray}
	\label{eq:ftca_opt2}
	\argmax_{\bB \in \bbR^{n \times k}} 
	\Tr\big( \bB^{\top} ( \frac{(1-\beta)}{n} \bK \bK+ \beta \bP) \bB \big) \\ 
	\,\,\,\textnormal{ s.t. } \,\,\,
	\Tr\big(\bB^{\top}  ( \delta \bK \bL \bK  +   \bQ + \bK) \bB \big) = 1, \nonumber
\end{eqnarray}
which results in Lagrangian
\begin{eqnarray}
	\label{eq:ftca_lagrang}
	J(\bB) = \Tr( \bB^{\top} ( \frac{(1-\beta)}{n}\bK \bK + \beta \bP)  \bB )   - \nonumber \\
	\Tr(  (\bB^{\top} ( \delta \bK \bL \bK  +  \bQ + \bK) \bB - \bI_k ) \mathbf{\Lambda}   ).
\end{eqnarray}
To solve (\ref{eq:ftca_opt}), set the first derivative $\frac{\partial J(\bB)} { \partial \bB} = 0$, inducing the generalized eigenproblem
\begin{eqnarray}
	\label{eq:ftca_eigenprob}
	\big( \frac{(1-\beta)}{n} \bK \bK + \beta \bP \big) \bB^{*} = 
	\big( \delta \bK \bL \bK + \bK +   \bQ \big) \bB^{*} \mathbf{\Lambda},
\end{eqnarray}
where $\mathbf{\Lambda} = \diag(\lambda_1, ..., \lambda_k)$ are the $k$ leading eigenvalues and $\bB = [\bb_1, ..., \bb_k]$ contains the corresponding eigenvectors.%
\footnote{In the implementation, a numerically more stable variant is obtained by using (\ref{eq:ftca_eigenprob}) using 
$\delta \bK \bL \bK + \bK + \bQ + \epsilon \bI$,
where $\epsilon>0$ is a fixed small constant.
}
Algorithm~\ref{alg:sca} provides a complete summary of SCA.

\begin{algorithm}[tb]
   \caption{
   Scatter Component Analysis}
   \label{alg:sca}
	\algorithmicrequire\; \\
		$\bullet$ Sets of training datapoints $S_u^d = \{ \bx^d_i\}_{i=1}^{n_d}, \forall d=1,\ldots,m$ and their corresponding matrices $\bX = \begin{bmatrix} \bX^1; \ldots ;\bX^m\end{bmatrix} \in \bbR^{n \times p}$, where $\bX^d = [\bx^d_1, \ldots, \bx^d_{n_d}]^{\top}$; \\
		$\bullet$ Training labels $\mathbf{y}^l = [y^1_1, \ldots, y^1_{n_1},\ldots,y^q_1, \ldots, y^q_{n_q}]^{\top} \in \bbR^{n}$; \\
		$\bullet$ Hyper-parameters $\beta,  \delta > 0$; kernel bandwidth $\sigma$; \\
		$\bullet$ Number of subspace bases $k$;
	\begin{algorithmic}[1]
		\STATE Construct kernel matrix $\bK$ from $\bX$, 
			matrices $\bL$, $\bP$ and $\bQ$ based on (\ref{eq:kernel}),  (\ref{eq:between_class_kernel}), (\ref{eq:within_class_kernel}), and (\ref{eq:big_pq});
		\STATE Apply the centering operation 
			$\bK \leftarrow \bK - \mathbf{1}_n \bK - \bK \mathbf{1}_n + \mathbf{1}_n \bK \mathbf{1}_n$,
			where $n = \sum_{d=1}^m n_d$ and $[1_n]_{ij} := \frac{1}{n}$;
		\STATE Obtain the transformation $\bB^{*}$ and its corresponding eigenvalues $\mathbf{\Lambda}$ by solving the generalized eigendecomposition problem in Eq. (\ref{eq:ftca_eigenprob}) and selecting the $k$ leading eigenvectors;
		\STATE Target feature extraction: Let $S_u = \bigcup_{d=1}^{m} S^d_u$ be the total training sample and $S^t_u$ be a target sample (for domain adaptation, $S^t_u \subset S_u$). Construct a kernel matrix $[\mathbf{K}^t]_{ij} = \kappa(\bx_i, \bt_j), \forall \bx_i \in S_u, \bt_j \in S^t_u$. The extracted features are given by $\bZ^t = \mathbf{K}^{t \top} \bB^* \mathbf{\Lambda}^{-\frac{1}{2}}$
	\end{algorithmic}
	\algorithmicensure\; \\
		$\bullet$ Optimal transformation matrix $\bB^{*} \in \bbR^{n \times k}$;\\
		$\bullet$ Feature matrix $\bZ^t \in \bbR^{n_t \times k}$.
\end{algorithm}

\subsection{Relation to other methods}
\label{s:related}
SCA is closely related to a number of feature learning and domain adaptation methods.
To see this, let us observe Lagrangian (\ref{eq:ftca_lagrang}).
Setting the hyper-parameters $\beta=\delta=0$ and $\bQ = \mathbf{0}$ in (\ref{eq:ftca_lagrang}) recovers KPCA. 
Setting $\beta =1$ and $\delta=0$ recovers the Kernel Fisher Discriminant (KFD) method \cite{Mika1999}. 
KFD with linear kernel is equivalent to Fisher's linear discriminant, which is the basis of a domain adaptation method for object detection proposed in \cite{Sun:BMVC2014}.

Setting $\beta=0$ and $\bQ = \mathbf{0}$ (that is, ignoring class separation) yields a new algorithm: unsupervised Scatter Component Analysis (uSCA), which is closely related to TCA. 
The difference between the two algorithms is that TCA constrains the total variance and regularizes the transform, 
whereas uSCA trades-off the total variance and constrains the transform (recall that $\bW^\top \bW$ should be small) motivated by Theorem~\ref{t:bound}. 
It turns out that uSCA consistently outperforms TCA in the case of domain adaptation, see Section~\ref{sec:exp1}.

Eliminating the term $\bB^\top \bK \bB $ from the denominator in (\ref{eq:ftca_opt}) from uSCA yields TCA~\cite{Pan2011}.
The semi-supervised extension SSTCA of TCA differs markedly from SCA. Instead of incorporating within- and between- class scatter into the objective function, SSTCA incorporates a term derived from the Hilbert-Schmidt Independence Criterion that maximizes the dependence of the embedding on labels. 

uSCA is essentially equivalent to unsupervised Domain Invariant Component Analysis (uDICA) in the case of two domains \cite{Muandet2013}. However, as for SSTCA, \emph{supervised} DICA incorporates label-information differently from SCA -- via the notion of a central subspace. In particular, supervised DICA requires that all data points are labeled, and so it cannot be applied in our experiments.

\vspace{-1em}
\subsection{Computational Complexity}
\label{s:runtime}
Here we analyze the computation complexity of the SCA algorithm.
Suppose that we have $m$ domains with $n_1, \ldots, n_m$ are the number of samples for each domain ($m > 2$ covers the domain generalization case).
Denote the total number of samples by $n = n_1 + \ldots + n_m$ and the number of leading eigenvectors by $k \ll n$.
Computing the matrices $\bK$, $\bL$, $\bP$, and $\bQ$ takes $O(n^2)$ (Line 1 at Algorithm~\ref{alg:sca}). 
Hence, the total complexity of SCA after solving the eigendecomposition problem (Line 2) takes $O(k n^2)$, or quadratic in $n$.
This complexity is similar to that of KPCA and Transfer Component Analysis~\cite{Pan2011}.

In comparison to Transfer Joint Matching (TJM)~\cite{Long2014a}, the prior state-of-the-art domain adaptation algorithm for object recognition, TJM uses an alternating eigendecomposition procedure in which $T$ iterations are needed.
Using our notation, the complexity of TJM is $O(Tkn^2)$, i.e., TJM is $T$ times slower than SCA.

\vspace{-1em}
\subsection{Hyper-parameter Settings}
Before reporting the detailed evaluation results, it is important to explain how SCA hyper-parameters were tuned.
The formulation of SCA described in Section~\ref{sec:sca} has four hyper-parameters: 
1) the choice of the kernel,
2) the number of subspace bases $k$,
3) the between-class and total scatters trade-off $\beta$, and
4) the domain scatter $\delta$,.
Tuning all those hyper-parameters using a standard strategy, e.g., a grid-search, might be impractical due to two reasons.
The first is of the computational complexity.
The second, which is crucial, is that cross-validating a large number of hyper-parameters may worsen the generalization on the target domain, since labeled samples from the target domain are not available.

Our strategy to deal with the issue is to reduce the number of tunable hyper-parameters. 
For the kernel selection, we chose the RBF kernel $\exp( \frac{- \| \ba - \bb \|_2^2}{\sigma^2} ), \forall \ba, \bb \in \cX$, where the kernel bandwidth $\sigma$ was set analytically to the median distance between samples in the aggregate domain following~\cite{Gretton:2012aa},
\begin{eqnarray}
 \sigma = \mathrm{median}(\| \ba - \bb \|_2^2), \forall \ba, \bb \in S^s \cup S^t.
\end{eqnarray}
For domain adaptation, $\delta$ was fixed at $1$. Thus, only two hyper-parameters remain tunable: $k$ and $\beta$.
For domain generalization, $\beta$ was set at 1, i.e., the total scatter was eliminated, and $\delta$ was allowed to be tuned  -- the number of tunable hyper-parameters remains unchanged.
The configuration is based on an empirical observation that setting $0 < \beta < 1$ is no better (if not worse) than $\beta=1$ in terms of both the cross-validation and test performance for domain generalization cases.
In all evaluations, we used 5-fold cross validation using source labeled data to find the optimal $k$ and $\beta$.
We found that this strategy is sufficient to produce good SCA models for both domain adaptation and generalization cases.

\section{Analysis of Adaptation Performance}
\label{sec:bound}
We derive a bound for domain adapation that shows how the MMD controls generalization performance in the case of the squared loss $\ell(y,y') = (y - y')^2$. 
Despite the widespread use of the MMD for domain adaptation~\cite{DAML:2011,DAM:2012,Long:2013aa,Long2014a,Pan2011}, to the best of our knowledge, this is the first generalization bound. 
The main idea is to incorporate the MMD (that is, domain scatter) into the adaptation bound proven for the \emph{discrepancy distance}~\cite{Mansour2009}. A generalization bound for domain generalization in terms of domain scatter is given in \cite{Muandet2013}, see remark~\ref{rem:gen_scatter}.

Let $\hyp := \{h : \cX \rightarrow \cY \}$  denote a hypothesis class of functions from $\cX$ to $\cY$ where $\cX$ is a compact set. Given a loss function defined over pairs of labels
$\ell: \cY \times \cY \rightarrow \bbR_{+}$ and a distribution $\bbD$ over $\cX$, let 
$\cL_{\bbD}\big(h, h'\big) = \expec_{x \sim \bbD}[\ell(h(x), h'(x))]$ denote the expected loss for any two hypotheses $h, h' \in \hyp$.
We consider the case where the hypothesis set $\hyp$ is a subset of an RKHS $\cH$.

We first introduce discrepancy distance, $\disc_{\hyp}(\bbP, \bbQ)$, which measures the difference between two distributions $\bbP$ and $\bbQ$.

\begin{defn}[\textbf{Discrepancy Distance}~\cite{Mansour2009}]
\label{def:disc}
Let $\hyp \subset \{ f : \cX \rightarrow \cY\}$ be a set of functions mapping from $\cX$ to $\cY$.
The \textbf{discrepancy distance} between two distributions $\bbP$ and $\bbQ$ over $\cX$ is defined by
\begin{eqnarray}
	\label{eq:discrepancy}
	\disc(\bbP, \bbQ) = \sup_{h,h' \in \hyp} \big| \cL_{\bbP}(h,h') - \cL_{\bbQ}(h,h') \big|
\end{eqnarray}
\end{defn}
The discrepancy is symmetric and satisfies the triangle inequality, but it does not define a distance in general: $\exists \bbP \neq \bbQ$ such that $\disc_{\hyp}(\bbP, \bbQ) = 0$~\cite{Mansour2009}. 

If we assume that we have a \emph{universal kernel}~\cite{Steinwart:2002aa,micchelli:06}, i.e. $\cH = C(\cX)$ as topological spaces, and the loss $\ell$ is the squared loss~\cite{Cortes:2014} then the discrepancy is a metric. The most important example of a universal kernel is the Gaussian RBF kernel, which is the kernel used in the experiments below.

The main step of the proof is to find a relationship between domain scatter and the discrepancy distance. We are able to do so in the special case where the kernel is universal and the loss is the mean-square error. The main technical challenge is that the discrepancy distance is quadratic in the hypotheses (involving terms of the form $h(x)^2$ and $h(x)h'(x)$) whereas the MMD is linear. We therefore need to bound the effects of the multiplication operator:

\begin{defn}[\textbf{Multiplication Operator}]
 \label{def:mulop}
 Let $C(\cX)$ be the space of continuous functions on the compact set $\cX$ equipped with the supremum norm $\| \cdot \|_\infty$. 
 Given $g \in C(\cX)$, define the multiplication operator as the bounded linear operator $\bM_{g} : C(\cX) \rightarrow C(\cX)$ given by
 \begin{equation}
    \bM_{g} (h)(x)  = g(x) h(x). \nonumber
 \end{equation}
\end{defn}

Note that a general RKHS is \emph{not} closed under the multiplication operator~\cite{Grunewalder:2013}. However, since the kernel is universal, it follows that $\cH$ is closed under multiplication since the space of continuous functions $C(\cX)$ is closed under multiplication. Moreover, we can define the sup-norm $\|\cdot\|_\infty$ on $\cH$ using its identification with $C(\cX)$.

The following Lemma upper bounds the norm of multiplication operator, which will be useful to prove our main theorem.
 \begin{lem}\label{lem:mult}
   Given $g$, $h\in \cH$, where $\cH$ is equipped with a universal kernel, it holds that
   $\| \bM_{g} (h) \|_\cH = \|g\cdot h\|_\cH \leq \|g\|_\infty\cdot \|f\|_\cH$.
 \end{lem}
 \begin{proof}
    Straightforward calculation. The Lemma requires a universal kernel since $\|g\cdot h\|_\cH$ is only defined if $g\cdot h\in \cH$.
 \end{proof}

We now show that the domain scatter of two distributions upper-bounds the discrepancy distance.
\begin{lem}[\textbf{Domain scatter bounds discrepancy}]
\label{t:disc_scatter}
  Let $\cH$ be an RKHS with a universal kernel.
  Suppose that $\ell(y,y')= (y - y')^2$ is the square loss, and consider the hypothesis set
\begin{equation*}
   \hyp = \{f\in\cH\,:\,\|f\|_\cH\leq 1 \text{ and }\|f\|_\infty\leq r\},
\end{equation*}
where $r > 0$ is a constant
Let $\bbP$ and $\bbQ$ be two domains over $\cX$. 
Then the following inequality holds:	
\begin{eqnarray}
  \underbrace{\discl(\bbP, \bbQ)}_{\text{discrepancy}} 
  \leq 
  \underbrace{8 r \sqrt{\Psi_{\phi}(\{ \mu_{\bbP}, \mu_{\bbQ}\})}}_{\text{domain scatter}}.
\end{eqnarray}
\end{lem}

\begin{proof} 
See Supplementary material. 
\end{proof}

Lemma \ref{t:disc_scatter} allows us to relate \emph{domain scatter} to generalization bounds for domain adaptation proven in \cite{Mansour2009}. 
Before stating the bounds, we introduce Rademacher complexity~\cite{Bartlett:2002}, which measures the degree to which a class of functions can fit random noise.
This measure is the basis of bounding the empirical loss and expected loss.

\begin{defn}[\textbf{Rademacher Complexity}]
\label{d:radem}
Let $G$ be a family of functions mapping from $\cX \times \cY$ to $[a,b]$ and $S = (z_1, ..., z_n) \in \cX \times \cY$ be a fixed sample of size $n$.
The empirical Rademacher complexity of $G$ with respect to the sample $S$ is
\begin{eqnarray}
 \hat{\radem}_S(G) = \expec_{\boldsymbol{\sigma}} \left[\sup_{g \in G} \frac{1}{n} \sum_{i=1}^{n} \sigma_i g(z_i) \right],
\end{eqnarray}
where $\boldsymbol{\sigma} = (\sigma_1,\ldots, \sigma_n)^{\top}$ are Rademacher variables, with $\sigma_i$s independent uniform random variables taking values in $\{ -1, +1\}$.
The \textbf{Rademacher complexity} over all samples of size $n$ is
\begin{eqnarray}
 {\radem}_n(G) = \expec_{S} \left [ \hat{\radem}_S(G) \right].
\end{eqnarray}
 
\end{defn}

The supplementary material discusses how to associate a family of functions to a loss function, and provides a useful Rademacher bound.  We now have all the ingredients to derive domain adaptation bounds in terms of domain scatter.

Let $f_{\bbP}$ and $f_{\bbQ}$ be the true labeling functions on domain $\bbP$ and $\bbQ$ respectively, 
and $h^*_{\bbP} := \argmin_{h \in \hyp} \cL_{\bbP}(h, f_{\bbP})$ and $h^*_{\bbQ} := \argmin_{h \in \hyp} \cL_{\bbQ}(h, f_{\bbQ})$ be the minimizers.
For a successful domain adaptation, we shall assume that $\cL_{\bbP}(h^*_{\bbP}, h^*_{\bbQ})$ is small.
The following theorem provides a domain adaptation bound in terms of scatter (recall that the MMD is a special case of scatter by Lemma~\ref{t:mmd}).

\begin{thm}[\textbf{Adaptation bounds with domain scatter}]
\label{t:adapt_bound}
Let $\hyp$ be a family of functions mapping from $\cX$ to $\bbR$,
$S^{\bbP}_{\cX} = (x^t_1,\ldots, x^t_{n_s}) \sim \bbP$ and $S^{\bbQ}_{\cX} = (x^t_1,\ldots, x^t_{n_t}) \sim \bbQ$ be a source and target sample respectively.
Let the rest of the assumptions be as in Lemma~\ref{t:disc_scatter} and Theorem~8 in the supplementary material.
For any hypothesis $h \in \hyp$, with probability at least $1 - \delta$, the following adaptation bound holds:
\begin{eqnarray}
  \label{eq:main_bound}
  \overbrace{\cL_{\bbQ}(h, f_{\bbQ})
  - \cL_{\bbQ}(h^{*}_{\bbQ}, f_{\bbQ})}^{\text{regret on target domain}} 
  \leq 
  \overbrace{\cL_{\hat{\bbP}}(h, h^*_{\bbP})}^{\text{empirical loss}} 
  + \overbrace{2q\hat{\radem}_{S^{\bbP}_{\cX}}(\hyp)}^{\text{Rademacher complexity}} \nonumber 
  \\
  + \underbrace{3 B \sqrt{\frac{\log{ \frac{2}{\delta}} }{2 n_t}}}_{O(1/\sqrt{\text{sample size})}} 
  + \underbrace{8r \sqrt{\Psi_\phi(\{\mu_\bbQ, \mu_\bbP\})}}_{\text{domain scatter}} 
  + \underbrace{\cL_{\bbP}(h^*_{\bbP},h^*_{\bbQ})}_{\text{deviation of optimal solns}}
\end{eqnarray}
\end{thm}
\begin{proof}
Fix $h \in \hyp$. Since the square loss is symmetric and obeys the triangle inequality, Theorem 8 in \cite{Mansour2009} (see supplementary material) implies that
\begin{eqnarray}
  \label{eq:triangle_loss}
 \cL_{\bbQ}(h, f_{\bbQ}) - \cL_{\bbQ}(h^{*}_{\bbQ}, f_{\bbQ}) \leq \cL_{\bbP}(h, h^*_{\bbP}) + \discl(\bbQ, \bbP) \nonumber \\ 
  + \cL_{\bbP}(h^*_{\bbP}, h^*_{\bbQ}).
\end{eqnarray}
The result then follows by Lemma~\ref{t:disc_scatter} combined with the Rademacher bound in the supplementary material.
\end{proof}
\noindent It is instructive to compare Theorem \ref{t:adapt_bound} above with Theorem 9 in \cite{Mansour2009}, 
which is the analog if we expand $\disc_l(\bbQ, \bbP)$ in (\ref{eq:triangle_loss}) with its empirical measure.
It is also straightforward to rewrite the bound in term of the \emph{empirical scatter} $\Psi_\phi(\{ \mu_{\hat{\bbP}}, \mu_{\hat{\bbQ}}\})$ by applying Theorem~\ref{t:bound}.

The significance of Theorem \ref{t:adapt_bound} is twofold. 
First, it highlights that the scatter $\Psi_\phi(\{ \mu_\bbP, \mu_\bbQ\} )$ controls the generalization performance in domain adaptation.
Second, the bound shows a direct connection between \emph{scatter} (also MMD) and the domain adaptation theory proposed in \cite{Mansour2009}.
Note that the bound might not be useful for practical purposes, since it is loose and pessimistic as they hold for all hypotheses and all possible data distributions.

\begin{rem}[\textbf{The role of scatter in domain generalization}]\label{rem:gen_scatter}
 Theorem~5 of \cite{Muandet2013} shows that the domain scatter (or, alternatively, the distributional variance) is one of the key terms arising in a generalization bound in the setting of domain generalization.  
\end{rem}

\vspace{-1em}

\section{Experiment I : Domain Adaptation}
\label{sec:exp1}
The first set of experiments evaluated the domain adaptation performance of SCA on synthetic data and real-world object recognition tasks.
The synthetic data was designed to understand the behavior of the learned features compared to other algorithms, 
whereas the real-world images were utilized to verify the performance of SCA.

The experiments are divided into two parts.
Section~\ref{sec:toy} visualizes performance on synthetic data.
Section~\ref{sec:exp1_vals} evaluates performance on a range of cross-domain object recognition tasks with a standard yet realistic hyper-parameter tuning.
Some additional results with a tuning protocol established in the literature are also reported in the supplementary material for completeness.

\begin{figure*}[htp]
	\centering
	\subfigure{\includegraphics[width=1.2in]{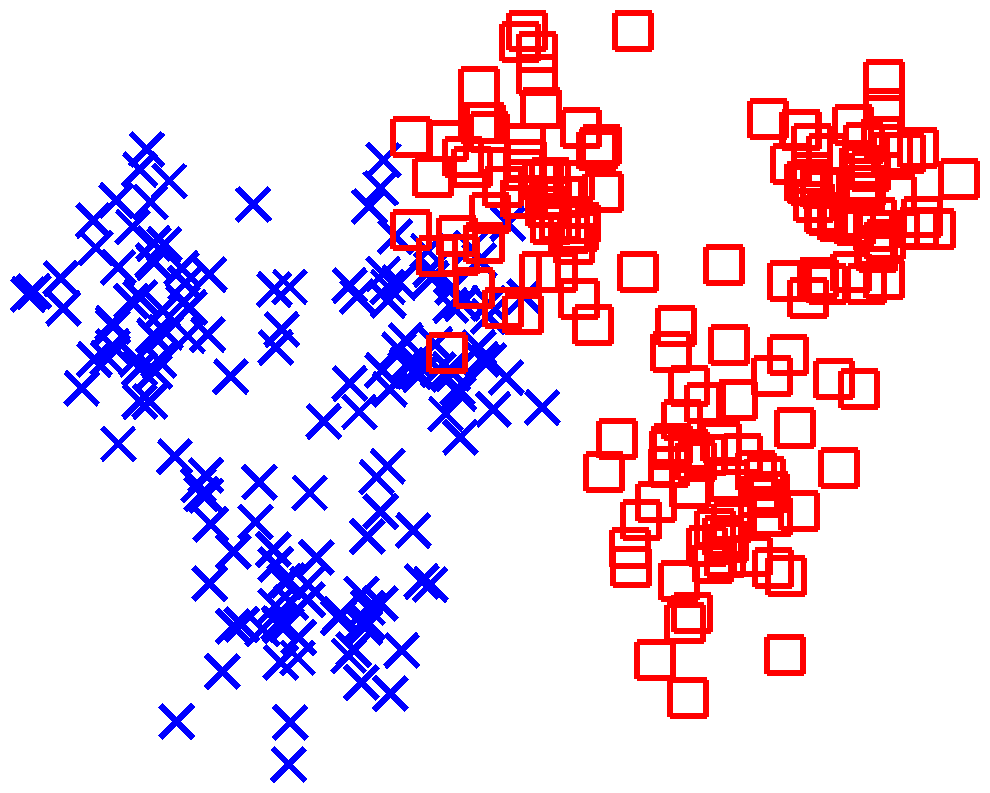}} \quad 
	\subfigure{\includegraphics[width=1.2in]{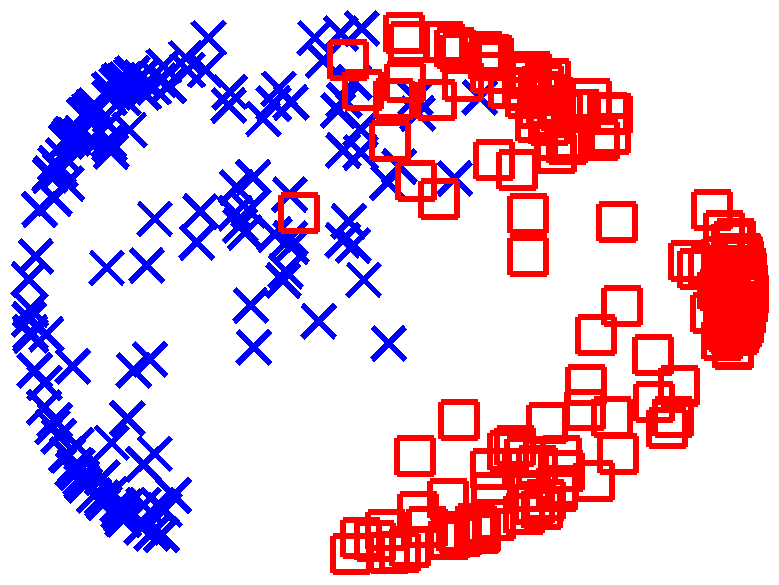}} \quad 
	\subfigure{\includegraphics[width=1.2in]{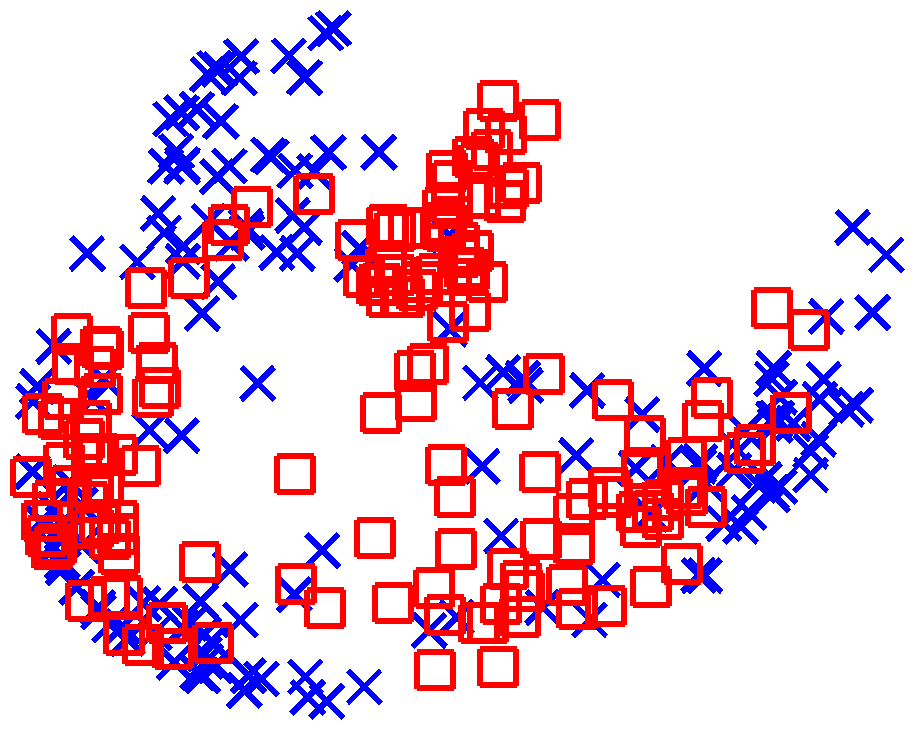}} \quad
	\subfigure{\includegraphics[width=1.2in]{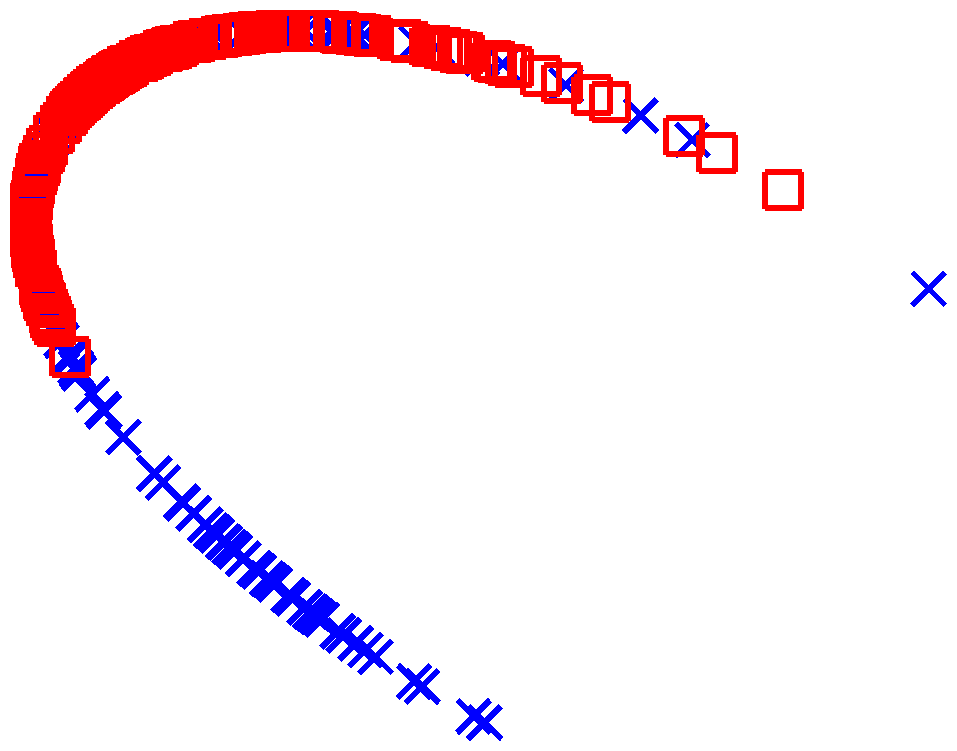}} \quad 
	\subfigure{\includegraphics[width=1.2in]{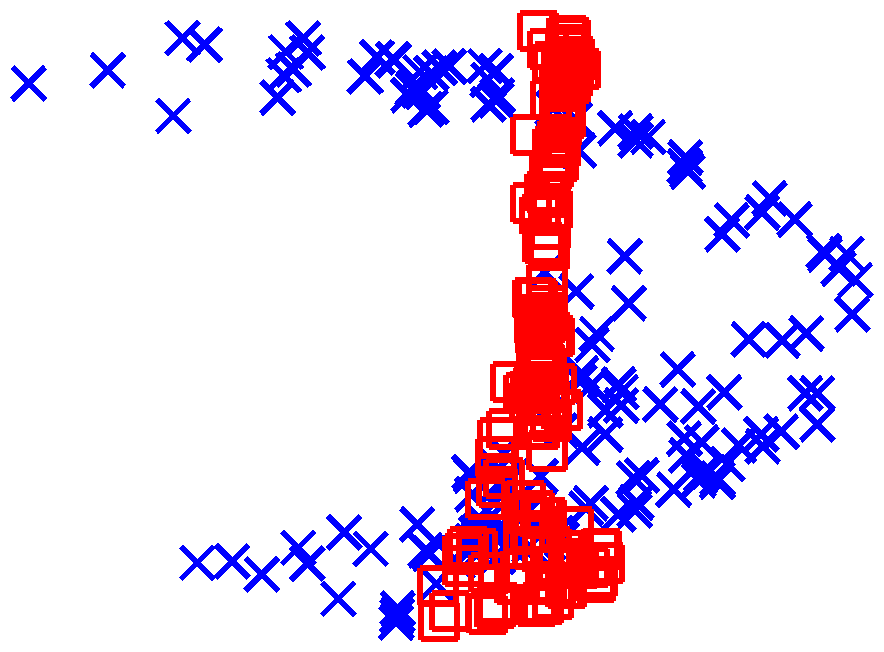}}
	\setcounter{subfigure}{0}
	\subfigure[{\small Raw ($28\%)$}]{\label{fig:toydataset}\includegraphics[width=1.2in]{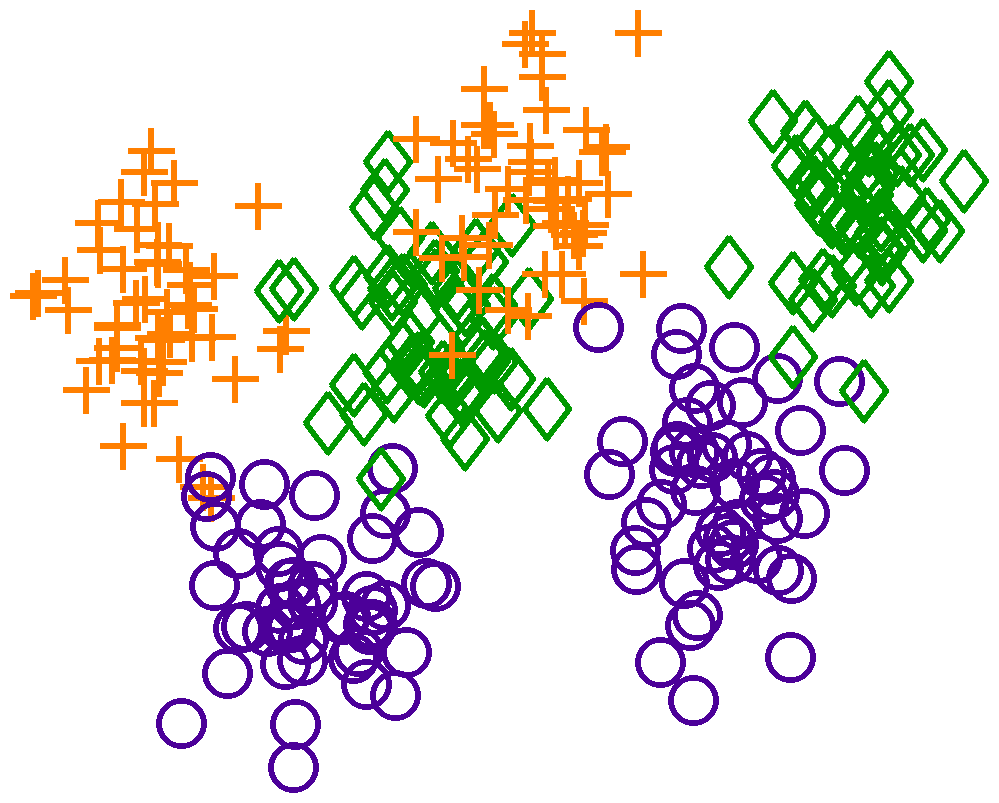}} \quad 
	\subfigure[{\small KPCA ($28\%$)}]{\label{fig:kpca_toy}\includegraphics[width=1.2in]{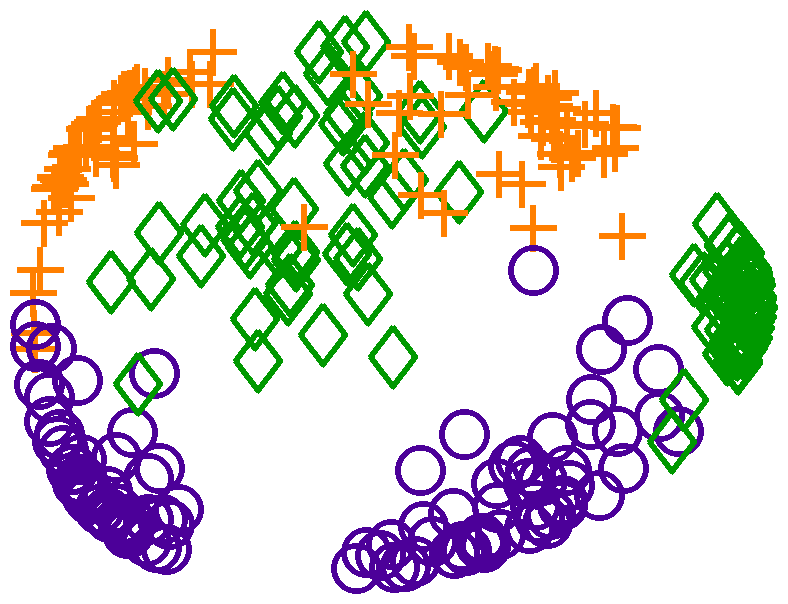}} \quad
	\subfigure[SSTCA ($36\%$)]{\label{fig:sstca_toy}\includegraphics[width=1.2in]{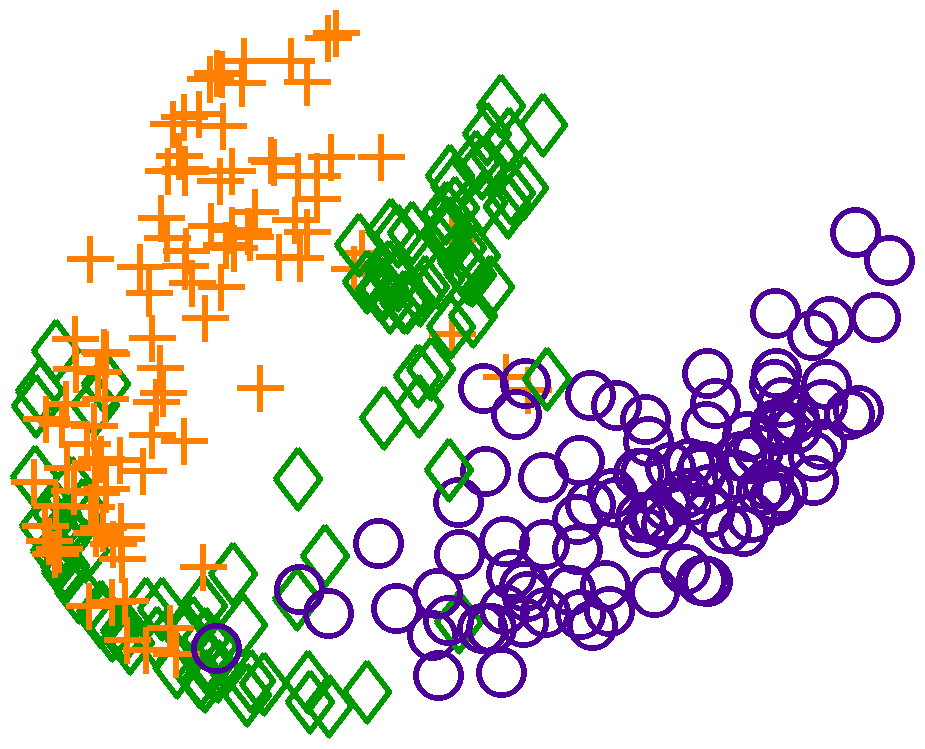}} \quad
	\subfigure[TJM ($44\%$)]{\label{fig:tjm_toy}\includegraphics[width=1.2in]{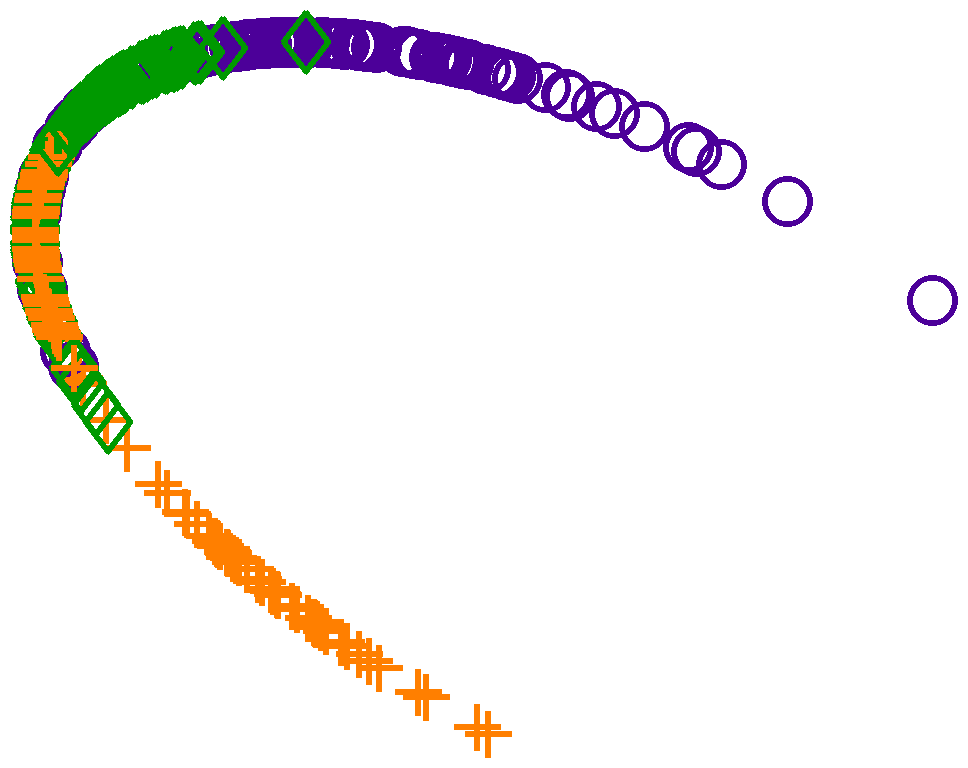}} \quad 
	\subfigure[SCA ($77\%$)]{\label{fig:sca_toy}\includegraphics[width=1.2in]{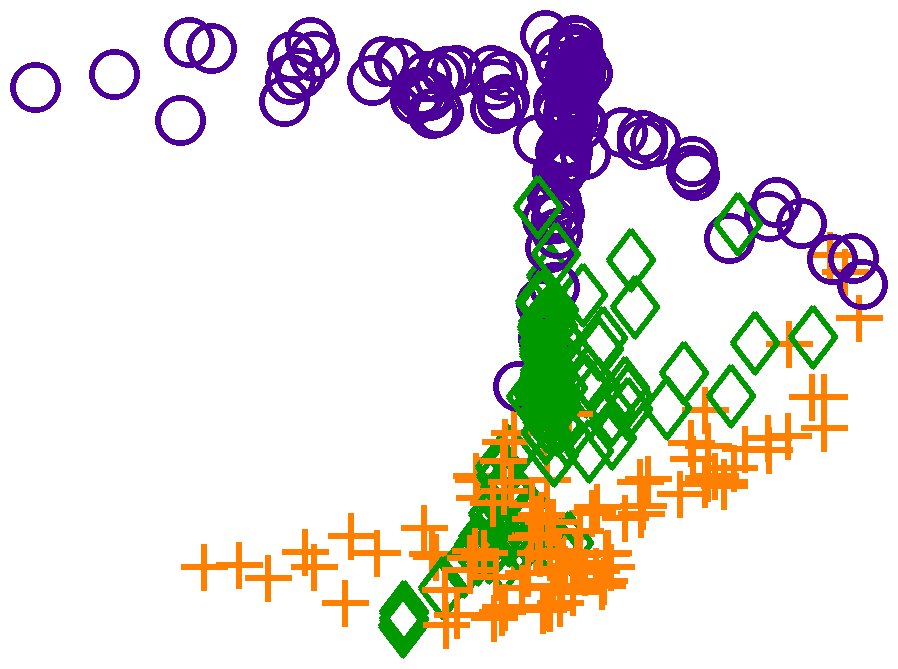}}
	\vspace{-1em}    
	\caption{\textbf{Visualization.}
	Projections of synthetic data onto the first two leading eigenvectors.
	Numbers in brackets indicate the classification accuracy on the target using 1-nearest neighbor (1NN).
	The top and bottom rows show the domains and classes respectively.
	}
	\label{fig:toy}
\end{figure*}
\vspace{-1em}
\subsection{Synthetic data}
\label{sec:toy}
Figure~1 depicts synthetic data that consists of two dimensional data points under three classes with six clusters.
The data points in each cluster were generated from a Gaussian distribution $x^c_{i} \sim \mathcal{N}(\mu^c, \sigma^c)$, 
where $\mu^c$ and $\sigma^c$ is the mean and standard deviation of the $c$-th cluster.
The RBF kernel $k(\ba ,\bb) = \exp(-\frac{\| \ba - \bb \|_2^2}{\sigma^2})$ was used for all algorithms.
All tunable hyper-parameters were selected according to 1-nearest neighbor's test accuracy.
We compare features extracted from Kernel Principal Component Analysis (KPCA), Semi-Supervised Transfer Component Analysis (SSTCA)~\cite{Pan2011}, Transfer Joint Matching (TJM)~\cite{Long2014a}, and SCA.

The top row of Figure~\ref{fig:toy} illustrates how the features extracted from the MMD-based algorithms (SSTCA, TJM, and SCA) reduce the domain mismatch. 
Red and blue colors indicate the source and target domains, respectively.
Good features for domain adaptation should have a configuration of which the red and blue colors are mixed.
This effect can be seen in features extracted from SSTCA, TJM, and SCA, which indicates that the domain mismatch is successfully reduced in the feature space.
In classification, domain adaptive features should also have a certain level of class separability.
The bottom row highlights a major difference between SCA and the other algorithms in terms of the class separability: the SCA features are more clustered with respect to the classes, with more prominent gaps among clusters.
This suggests that it would be easier for a simple function to correctly classify SCA features.

\subsection{Real world object recognition}
\label{sec:exp1_vals}
We summarize the complete domain adaptation results over a range of cross-domain object recognition tasks.
Several real-world image datasets were utilized such as handwritten digits (MNIST~\cite{LeCun:1998aa} and USPS~\cite{usps1994}) and general objects (MSRC~\cite{msrc2005}, VOC2007~\cite{pascal-voc-2007}, Caltech-256~\cite{Griffin2007}, Office~\cite{Saenko:2010aa}).
Three cross-domain pairs were constructed from these datasets: USPS+MNIST, MSRC+VOC2007, and Office+Caltech.

\subsubsection{Data Setup}
\label{sec:exp1_datasetup}
\textbf{The USPS+MNIST}
pair consists of raw images subsampled from datasets of handwritten digits.
MNIST contains 60,000 training images and 10,000 test images of size $28 \times 28$.
USPS has 7,291 training images and 2,007 test images of size $16 \times 16$ \cite{LeCun:1998aa} .
The pair was constructed by randomly sampling 1,800 images from USPS and 2,000 images from MNIST.
Images were uniformly rescaled to size $16 \times 16$ and encoded into feature vectors representing the gray-scale pixel values.
Two \textsc{source} $\rightarrow$ \textsc{target} classification tasks were constructed: \textsc{usps} $\rightarrow$ \textsc{mnist} and \textsc{mnist} $\rightarrow$ \textsc{usps}.

\textbf{The MSRC+VOC2007}
pair consist of 240-dimensional images that share 6 object categories: ``aeroplane'', ``bicycle'',``bird'', ``car'', ``cow'', and ``sheep'' taken from the MSRC and VOC2007~\cite{pascal-voc-2007} datasets.
The pair was constructed by selecting all 1,269 images in MSRC and 1,530 images in VOC2007.
As in \cite{Long:2013aa}, features were extracted from the raw pixels as follows.
First, images were uniformly rescaled to be 256 pixels in length.  
Second, 128-dimensional dense SIFT (DSIFT) features were extracted using the VLFeat open source package~\cite{vedaldi08vlfeat}. 
Finally, a 240-dimensional codebook was created using K-means clustering to obtain the codewords.

\textbf{The Office+Caltech}
consists of 2,533 images of ten categories (8 to 151 images per category per domain), that forms four domains:
($A$) {\sc amazon}, ($D$) {\sc dslr}, ($W$) {\sc webcam}, and ($C$) {\sc caltech}.
{\sc amazon} images were acquired in a controlled environment with studio lighting.
{\sc dslr} consists of high resolution images captured by a digital SLR camera in a home environment under natural lighting.
{\sc webcam} images were acquired in a similar environment to {\sc dslr}, but with a low-resolution webcam.
Finally, {\sc caltech} images were collected from Google Images~\cite{Griffin2007}.
Taking all possible source-target combinations yields 12 cross-domain datasets denoted by 
$A \rightarrow W, A \rightarrow D, A \rightarrow C, \ldots, C \rightarrow D$.
We used two types of extracted features from these datasets that are publicly available: SURF-BoW\footnote{\url{http://www-scf.usc.edu/~boqinggo/da.html}}~\cite{Saenko:2010aa} and $\textnormal{DeCAF}_6$\footnote{\url{http://vc.sce.ntu.edu.sg/transfer_learning_domain_adaptation/domain_adaptation_home.html}}~\cite{Donahue:2014aa}.
\textbf{SURF-BoW} features were extracted using SURF~\cite{Bay:2008aa} and quantized into 800-bin histograms with codebooks computed by K-means on a subset of {\sc amazon} images.
The final histograms were standardized to have zero mean and unit standard deviation in each dimension.
\textbf{Deep Convolutional Activation Features (DeCAF)}
were constructed by \cite{Donahue:2014aa} using the deep convolutional neural network architecture in \cite{Krizhevsky:2012aa}.
The model inputs are the mean-centered raw RGB pixel values that are forward propagated through 5 convolutional layers and 3 fully-connected layers.
We used the outputs from the 6th layer as the features, leading to $4,096$ dimensional $\text{DeCAF}_6$ features.

\begin{table*}[!htb]
    \caption{Accuracy \% on the USPS+MNIST and MSRC+VOC2007 datasets.
    Bold-red and bold-black indicate the best and second best performance.
    }
    \vspace{-1em}
    \centering
    \begin{tabular}{| c || c | c | c | c | c | c | c | c | c | c |}
    \hline
    Dataset & Raw & KPCA & TCA & SSTCA &GFK & TSC & SA &  TJM & uSCA & SCA \\
    \hline
    
    \textsc{usps} $\rightarrow$ \textsc{mnist} & 
      $34.80$ & $42.55$  & $41.75$ & $40.07$ & $43.50$ & $40.95$ & $41.50$ & {\color{red}$\mathbf{52.65}$} &$44.86$ &  $\mathbf{48.00}$\\
      
    \textsc{mnist} $\rightarrow$ \textsc{usps} & 
      $63.06$ & $62.61$ & $59.44$ & $60.13$ & $61.22$ & $59.56$ & $63.95$ & $62.00$ & $\mathbf{64.67}$ & {\color{red} $\mathbf{65.11}$} \\
    
    \textsc{msrc} $\rightarrow$ \textsc{voc} & 
      $28.63$ & $29.35$ & $31.70$ & $30.95$ & $30.63$ & $28.80$ & $30.90$ &$32.48$ & {\color{red} $\mathbf{33.14}$} & $\mathbf{32.75}$ \\
    
    \textsc{voc} $\rightarrow$ \textsc{msrc} & 
      $41.06$ & $47.12$ & $45.78$ & $46.06$ & $44.47$ & $40.58$ & $46.88$ & $46.34$ &  {\color{red} $\mathbf{49.80}$} & $\mathbf{48.94}$ \\
    
    \hline
    \hline
    
    Avg. & $41.89$ & $45.51$ & $44.67$ & $44.30$ & $44.96$ & $42.47$ & $45.81$ & $\mathbf{48.37}$ & $48.12$ & {\color{red} $\mathbf{48.70}$}\\
    
    \hline
    \end{tabular}
    \label{tab:exp1_results_cv}
\end{table*}

\begin{table*}[!htb]
\caption{Accuracy \% on the Office+Caltech images with SURF-BoW features.
    1NN was used as the base classifier.}
    \vspace{-1em}
\centering
\begin{tabular}{| c || c | c | c  | c | c | c | c | c | c | c | }
\hline
Dataset           & Raw     & KPCA      &  TCA     & SSTCA &  GFK & TSC & SA      & TJM     & uSCA    & SCA    \\
\hline
\hline
$A \rightarrow W$ & $29.83$ & $31.86$ & $25.08$ & $28.15$  & {\color{red}$\mathbf{39.32}$} & $32.95$ & $\mathbf{37.63}$ & $33.56$ & $32.88$ & $33.90$                                 \\
$A \rightarrow D$ & $25.48$ & $33.76$  & $31.21$ & $32.25$  & $28.66$ & $33.14$ & $\mathbf{34.49}$ & {\color{red} $\mathbf{35.67}$} & $33.85$ &  $34.21$                                \\
$A \rightarrow C$ & $26.00$ & $37.04$  & $33.93$ & $32.48$  & {\color{red}$\mathbf{39.27}$} & $34.42$ & $37.80$  & $37.58$ & $37.13$ & $\mathbf{38.29}$                                 \\
$W \rightarrow A$ & $22.96$ & $29.44$  & $22.86$ & $25.56$ & $\mathbf{34.03}$ & $27.91$ & {\color{red}$\mathbf{34.34}$} & $29.85$ & $30.41$ & $30.48$                                 \\
$W \rightarrow D$ & $59.24$ & $\mathbf{89.81}$   & $65.61$   &  $80.81$ & $84.71$ & $83.27$ & $80.89$   & $86.62$ & $\mathbf{89.81}$ & {\color{red}$\mathbf{92.36}$}                                   \\
$W \rightarrow C$ & $19.86$ & $27.60$  & $23.06$ & $25.39$ & $28.76$ & $28.62$ & $28.76$ & $\mathbf{29.72}$ & $28.52$ & {\color{red}$\mathbf{30.63}$}                                 \\
$D \rightarrow A$ & $28.50$ & $31.00$  & $30.17$ & $29.16$ & $32.25$ & $31.00$ & {\color{red}$\mathbf{34.24}$} & $30.06$ & $31.00$ & $\mathbf{33.72}$                                 \\
$D \rightarrow W$ & $63.39$ & $84.41$  & $64.75$ & $78.90$ & $80.34$ & $85.13$ & $82.37$ & {\color{red}$\mathbf{90.85}$} & $84.41$ & $\mathbf{88.81}$                                 \\
$D \rightarrow C$ & $26.27$ & $27.78$  & $28.05$ & $28.05$ & $29.12$ & $28.59$ & $31.17$ & $\mathbf{30.72}$ & $27.78$ & {\color{red}$\mathbf{32.32}$}                                 \\
$C \rightarrow A$ & $23.70$ & $40.40$  & $41.02$ & $40.67$ & $41.75$ & $39.21$ & $41.34$ & {\color{red}$\mathbf{45.41}$} & $40.40$ & $\mathbf{43.74}$                                 \\
$C \rightarrow W$ & $25.76$ & $31.53$  & $23.39$ & $26.62$ & {\color{red}$\mathbf{36.61}$} & $29.97$ & $32.20$ & $\mathbf{33.90}$ & $29.15$ & $33.56$                                 \\
$C \rightarrow D$ & $25.48$ & $40.76$  & $34.49$ & $36.45$ & $40.13$ & $35.37$ & {\color{red}$\mathbf{42.86}$} & $40.31$ & $\mathbf{42.04}$ &  $39.49$                                 \\
\hline
\hline
Avg.              & $31.37$ & $42.12$  & $35.29$ & $38.17$ & $42.91$ & $40.80$ & $43.21$ & $\mathbf{43.67}$ & $42.28$ & {\color{red} $\mathbf{44.29}$} \\
\hline
\end{tabular}
\label{tab:office_results_surf_cv}
\end{table*}

\begin{table*}[!htb]
\caption{Accuracy \% on the Office+Caltech images with $\textnormal{DeCAF}_6$ features.
    1NN was used as the base classifier.}
    \vspace{-1em}
\centering
\begin{tabular}{| c || c | c | c  | c | c | c | c | c | c | c | }
\hline
Dataset           & Raw     & KPCA       & TCA     & SSTCA & GFK & TSC & SA      & TJM     & uSCA    & SCA    \\
\hline
\hline
$A \rightarrow W$ & $57.29$ & $67.80$  & $71.86$ & $70.73$ & $68.47$ & $68.54$ & $68.81$ & $72.54$ & $\mathbf{73.22}$ & {\color{red}$\mathbf{75.93}$}                                 \\
$A \rightarrow D$ & $64.97$ & $80.89$  & $78.34$ & $80.13$ & $79.62$ & $79.33$ & $78.34$ & {\color{red}$\mathbf{85.99}$} & $79.43$ &  $\mathbf{85.35}$                                 \\
$A \rightarrow C$ & $70.35$ & $74.53$  & $74.18$ & $72.25$ & $76.85$ & $74.91$ & {\color{red}$\mathbf{80.05}$} & $78.45$ & $74.62$ & $\mathbf{78.81}$                                 \\
$W \rightarrow A$ & $62.53$ & $69.42$  & $79.96$ & $75.65$ & $75.26$ & $73.65$ & $77.77$ & $\mathbf{82.46}$ & $79.52$ & {\color{red}$\mathbf{86.12}$}                                 \\
$W \rightarrow D$ & $\mathbf{98.73}$ & {\color{red}$\mathbf{100}$}   & {\color{red}$\mathbf{100}$} & {\color{red}$\mathbf{100}$}   & {\color{red}$\mathbf{100}$} &   {\color{red}$\mathbf{100}$} & {\color{red}$\mathbf{100}$}   & {\color{red}$\mathbf{100}$} & {\color{red}$\mathbf{100}$}   & {\color{red}$\mathbf{100}$}                                   \\
$W \rightarrow C$ & $60.37$ & $65.72$  & $72.57$ & $69.30$ & $74.80$ & $73.27$ & $\mathbf{74.89}$ & {\color{red}$\mathbf{79.61}$} & $72.81$ & $74.80$                                 \\
$D \rightarrow A$ & $62.73$ & $80.06$  & $88.20$ & $87.30$ & $85.80$ & $84.26$ & $82.67$ & {\color{red}$\mathbf{91.34}$} & $88.71$ & $\mathbf{89.98}$                                 \\
$D \rightarrow W$ & $89.15$ & $98.31$  & $97.29$ & $97.56$ & $\mathbf{98.64}$ & $97.29$ & {\color{red}$\mathbf{99.32}$} & $98.31$ & $98.31$ & $\mathbf{98.64}$                                 \\
$D \rightarrow C$ & $52.09$ & $75.16$  & $73.46$ & $74.45$ & $74.09$ & $76.11$ & $75.69$ & {\color{red} $\mathbf{80.77}$} & $74.98$ & $\mathbf{78.09}$                                 \\
$C \rightarrow A$ & $85.70$ & $88.73$  & $89.25$ & $88.90$ & $88.41$ & $88.52$ & $\mathbf{89.46}$ & {\color{red}$\mathbf{89.67}$} & $88.52$ & $\mathbf{89.46}$                                 \\
$C \rightarrow W$ & $66.10$ & $77.29$  & $80.00$ & $81.22$ & $\mathbf{80.68}$& $80.15$ & $75.93$ & $\mathbf{80.68}$ & $76.27$ & {\color{red} $\mathbf{85.42}$}                                 \\
$C \rightarrow D$ & $74.52$ & $86.62$  & $83.44$ & $84.56$ & $84.56$ & $86.62$ & $83.44$ & $\mathbf{87.26}$ & $86.62$ &  {\color{red}$\mathbf{87.90}$}                               \\
\hline
\hline
Avg.              & $70.38$ & $80.38$  & $82.38$ & $81.84$ & $82.44$ & $81.72$ & $82.20$ & $\mathbf{85.59}$ & $82.75$ & {\color{red} $\mathbf{85.88}$} \\
\hline
\end{tabular}
\label{tab:office_results_decaf_cv}
\end{table*}

\vspace{-0.5em}
\subsubsection{Baselines and Protocol}
\label{sec:exp1_baseline}
We compared the classification performance of the following algorithms:
1) a classifier on raw features (Raw),
2) KPCA,
3) Transfer Component Analysis (TCA)~\cite{Pan2011},
4) SSTCA,
5) Geodesic Flow Kernel (GFK)~\cite{Gong:2012aa},
6) Transfer Sparse Coding (TSC)~\cite{Long:2013aa},
7) Subspace Alignment (SA)~\cite{Fernando:2013aa},
8) TJM~\cite{Long2014a},
9) unsupervised Scatter Component Analysis, and
10) SCA.
For a realistic setting, the tunable hyper-parameters were selected via 5-fold cross validation, according to labels from source domains only.

The above feature learning algorithms were evaluated on three different classifiers: 
1) 1-nearest neighbor (1NN), 2) support vector machines with linear kernel (L-SVM) \cite{Boser:1992aa}, and 3) domain adaptation machines (DAM) \cite{DAM:2012}.
1NN and L-SVM are the standard off-the-shelf classifiers, while DAM is specifically designed for domain adaptation. 
DAM is an extension of SVM that incorporates a domain-dependent regularization to encourage the target classifier sharing similar prediction values with the source classifiers.
We also utilize the linear kernel for DAM.

\subsubsection{Classification Accuracy with 1-Nearest Neighbor}
\label{sec:exp1_acc}
We first report the classification accuracy of the competing algorithms according to 1NN classifier.
The goal is to clearly highlight the adaptation impact induced purely from the representations,
since 1NN basically just measures the distance between features.

Table~\ref{tab:exp1_results_cv} summarizes the classification accuracy on the USPS+MNIST and MSRC+VOC2007 pairs.
We can see that SCA is the best model on average, while the prior state-of-the-art TJM is the second best.
Other domain adaptation algorithms (TCA, SSTCA, GFK, and TSC) do not perform well, even worse than one without adaptation strategy: KPCA.
Surprisingly, the unsupervised version of our algorithm, uSCA, has the highest accuracy on two MSRC+VOC2007 cases.
This indicates that the label incorporation does not help improve domain adaptation on the MSRC+VOC2007, while it clearly does on the USPS+MNIST.
Furthermore, SCA and uSCA always provide improvement over the raw features, while other algorithms, including TJM, fail to do so in $\textsc{mnist} \rightarrow \textsc{usps}$ case.

Surprisingly, SSTCA, which also incorporates label information during training, does not perform competitively. 
The first possible explanation is that SCA \emph{directly} improves class separability, whereas SSTCA maximizes a dependence criterion that relates \emph{indirectly} to separability.
The second is that SSTCA incorporates the manifold regularization that requires a similarity graph, i.e., affinity matrix. 
This graph is parameterized by k-nearest neighbor with $l_2$ distance, which might not be suitable in these cases.

The results on the Office+Caltech pair are summarized in Table~\ref{tab:office_results_surf_cv} (SURF-BoW) and Table~\ref{tab:office_results_decaf_cv} ($\textnormal{DeCAF}_6$).
In general, $\textnormal{DeCAF}_6$ induces stronger discriminative performance than SURF-BoW features, since $\textnormal{DeCAF}_6$ with 1NN only has already provided significantly better performance.
SCA consistently has the best average performance on both features, slightly better than the prior state-of-the-art, TJM. 
On SURF-BoW, SCA is the best model on 3 out of 12 cases and the second best on other 4 cases.
The trend on $\textnormal{DeCAF}_6$ is better -- SCA has the best performance on 5 out of 12 cases, while comes second on other 6 cases.
Although the closest competitor, TJM, has the highest number of individual best cross-domain performance, it requires higher computational complexity than SCA, see Section~\ref{sec:exp1_runtime} below.

Recall that the algorithms' hyper-parameters used to produce all the above results were tuned using labels from the source domain only.
This is the only valid tuning protocol for the unsupervised domain adaptation setting.
Nevertheless, some of the best results established in the literature were obtained using the hyper-parameter tuning on target labels.
For completeness, we also report the results under this tuning-on-target protocol in the supplementary material.

\begin{figure*}[htp]
	\centering
	\subfigure[MNIST+USPS, MSRC+VOC]{\includegraphics[width=2.3in]{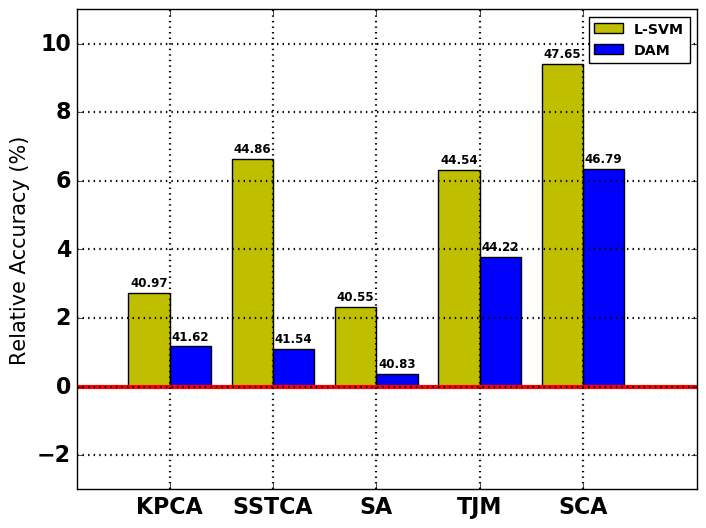}} \label{fig:digobj}\quad
	\subfigure[Office+Caltech (SURF-BoW)]{\includegraphics[width=2.3in]{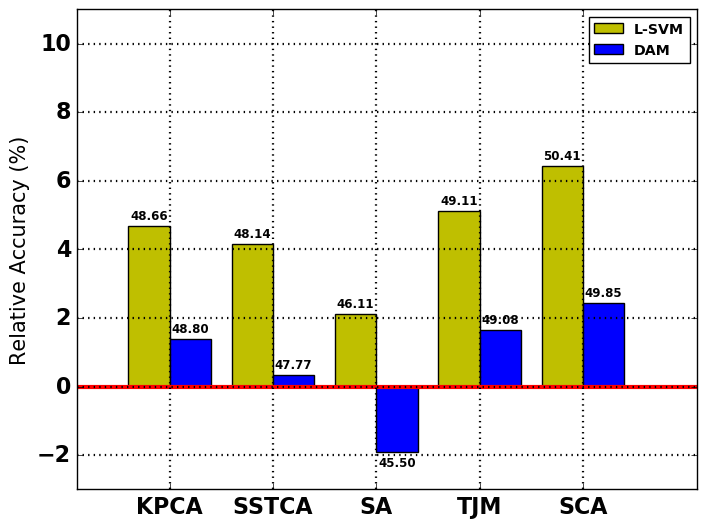}} \label{fig:office_surf} \quad
	\subfigure[Office+Caltech ($\textnormal{DeCAF}_6$)]{\includegraphics[width=2.3in]{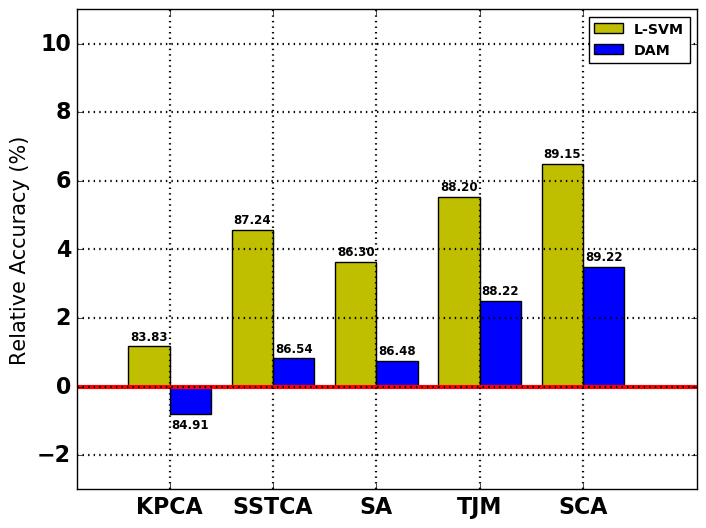}} \label{fig:office_decaf}
	\caption{L-SVM and DAM average performance accuracy (\%) relative to the performance on Raw features. The numbers on the top or bottom of the bars show the absolute accuracy. 
	The red line indicates the Raw baseline performance, see Table \ref{tab:raw_summ} for the exact numbers.
	}
	\label{fig:svm_dam}
\end{figure*}

\vspace{-1em}
\begin{table}[!htb]
\centering
\caption{Average accuracy (\%) on Raw features.}
\vspace{-1em}
\begin{tabular}{| c || c | c | c  | }
\hline
Dataset & 1-NN & L-SVM & DAM \\
\hline
MNIST+USPS, MSRC+VOC & $\mathbf{41.89}$ & $38.23$ & $40.45$ \\
Office+Caltech (SURF-BoW) & $31.37$ & $43.98$ & $\mathbf{47.42}$ \\
Office+Caltech ($\textnormal{DeCAF}_6$) & $70.38$ & $82.66$ & $\mathbf{85.72}$ \\
\hline
\end{tabular}
\label{tab:raw_summ}
\end{table}

\vspace{-1em}
\subsubsection{Classification Accuracy with L-SVM and DAM}
Next we report the results with L-SVM and DAM as the base classifiers for the feature learning algorithms.
For succinctness, we compare the performance of five algorithms: KPCA, SSTCA, SA, TJM, and SCA, presented in Figure \ref{fig:svm_dam}.
The bar chart shows the average accuracies relative to the performance on Raw features (indicated by line $y=0$ in red); the numbers alongside the bars indicate the absolute accuracies.
Table \ref{tab:raw_summ} summarizes the absolute accuracies on Raw features.

In general, all feature learning algorithms rectify the domain adaptation performances over Raw features, except in two cases: 
SA on the Office+Caltech with SURF-BoW features and KPCA on the Office+Caltech with $\textnormal{DeCAF}_6$ features.
Considering the absolute accuracies, we find that the best average performances on each dataset are still provided by SCA, a similar trend as in the 1NN results.
This confirms the effectiveness of SCA regardless of the classifier choice, at least, among 1NN, L-SVM, and DAM.

Let us now compare the absolute average performance of L-SVM and DAM with the performance of 1NN. 
L-SVM and DAM evidently provide a considerable performance improvement only on the Office+Caltech dataset.
Their performances on less powerful features, that is, the features extracted from the MNIST+USPS and MSRC+VOC, are even worse than 1NN.
A useful lesson from this finding is that one should make use better features to take the real benefit of more advanced classifiers in the context of domain adaptation.

Finally, we seek to investigate the performance impact induced by DAM in comparison to L-SVM.
DAM is expected to provide a better performance, since it is specifically designed for domain adaptation.
From Table \ref{tab:raw_summ} we can see that DAM outperforms L-SVM when operating on Raw features.
Surprisingly, that is not always the case when a feature learning algorithm is applied.
Moreover, L-SVM always produces higher performance gain relative to Raw features.than DAM.
This could be attributed to overfitting considering that DAM has more hyper-parameters than L-SVM.
That is, combining DAM with a feature learning algorithm complicates the whole processs -- recall that the hyper-parameter selection is based on a validation on source data.

\subsubsection{Runtime Performance}
\label{sec:exp1_runtime}
Table~\ref{tab:runtime} compares the average runtime performance of SCA with some other algorithms: KPCA, TCA, SA, TSC, and TJM, on the MNIST+USPS, MSRC+VOC, and Office+Caltech (with $\textnormal{DeCAF}_6$).
All algorithms were executed with MATLAB R2014b by a machine with Intel Core i5-240 CPU, Arch Linux 64-bit OS, and 4GB RAM.
Note that KPCA, TCA, and SCA basically utilizes the same optimization procedure: a single iteration of the eigenvalue decomposition.
TJM requires several iterations of the eigenvalue decomposition with an additional gradient update in each iteration, 
while TSC solves the dictionary learning and sparse coding with an iterative procedure.

In general, SCA is significantly faster than TJM, the closest competitor in accuracy, and TSC.
Specifically, SCA is $3$ to $6\times$ faster than TJM, and $>50\times$ faster than TSC.
SCA runs at about the same speed as TCA on the MNIST+USPS and MSRC+VOC, and the same speed as SA on the Office+Caltech.
In several other cases, SCA performs slower than KPCA, TCA, and SA.
Note that KPCA, TCA, and SA are less competitive in accuracy compared to SCA and TJM so that the runtime gap is less interesting to be concerned about.

\vspace{-1em}
\begin{table}[!htb]
	\caption{Average runtime (seconds) over all cross-domain tasks in each dataset.}
	\vspace{-1em}
	\centering
	\begin{tabular}{| c || c | c | c  | c | c | c | }
	\hline
	Dataset & KPCA & TCA & SA & TJM & TSC &  SCA \\
	\hline
	MNIST+USPS & $9.83$ & $41.55$ & $0.75$ & $269.44$ & $3072.25$ & $42.74$ \\
	MSRC+VOC & $3.23$ & $20.92$ & $0.69$ & $127.35$ & $2051.05$ & $36.86$ \\
	Office+Caltech & $0.84$ & $3.03$ & $8.35$ & $29.65$ & $1070.98$ & $8.82$ \\
	\hline
	\end{tabular}
	\label{tab:runtime}
\end{table}


\vspace{-1em}
\section{Experiment II : Domain Generalization} 
\label{sec:exp2}
In the second set of experiments, we show that our proposed algorithm is also applicable for domain generalization 
and achieves state-of-the-art performance on some object and action recognition datasets.
We evaluated our algorithms on three cross-domain datasets: the  VLCS, Office+Caltech, and IXMAS~\cite{Ixmas2006}.

\vspace{-0.5em}
\subsection{Data setup}
\label{exp2:data_setup}
The first cross-domain dataset, which we refer to as the \textbf{VLCS}  consists of images from PASCAL VOC2007 (V)~\cite{pascal-voc-2007}, LabelMe (L)~\cite{LabelMe}, Caltech-101 (C)~\cite{Griffin2007}, and SUN09 (S)~\cite{SUN09} datasets, each of which represents one domain.
These datasets share five object categories: 
\emph{bird}, \emph{car}, \emph{chair}, \emph{dog}, and \emph{person}.
Each domain in the VLCS dataset was divided into a training set ($70\%$) and a test set ($30\%$) by random selection from the overall dataset.
The detailed training-test configuration for each domain is summarized in the supplementary material.
We employed the $\textnormal{DeCAF}_6$ features \cite{Donahue:2014aa} with dimensionality of 4,096 as inputs to the algorithms.
These features are publicly available.\footnote{\url{http://www.cs.dartmouth.edu/~chenfang/proj_page/FXR_iccv13/index.php}}

The second cross-domain dataset is the \textbf{Office+Caltech} dataset, see Section~\ref{sec:exp1_datasetup} for a detailed explanation about this dataset.
We also used $\textnormal{DeCAF}_6$ features extracted from this dataset.\footnote{\url{http://vc.sce.ntu.edu.sg/transfer_learning_domain_adaptation/}}
The third dataset is the \textbf{IXMAS} dataset~\cite{Ixmas2006} that contains videos of the 11 actions, recorded with different actors, cameras, and viewpoints.
This dataset has been used as a benchmark for evaluating human action recognition models.
To simulate the domain generalization problem, we followed the setup proposed in \cite{Xu2014}: only frames from five actions were utilized (\emph{check watch}, \emph{cross arms}, \emph{scratch head}, \emph{sit down}, and \emph{get up}) with domains represented represented by camera viewpoints (Cam 0, Cam 1, ..., Cam 4). 
The task is to learn actions from particular camera viewpoints and classify actions on unseen viewpoints.
In the experiment, we used the dense trajectories features~\cite{DTrajectory2013} extracted from the raw frames and applied K-means clustering to build a codebook with 1,000 clusters 
for each of the five descriptors: trajectory, HOG, HOF, MBHx, MBHy. 
The bag-of-words features were then concatenated forming a 5,000 dimensional features for each frame.

\begin{table*}[!htb]
\caption{The groundtruth 1NN accuracy $\%$ of five-class classification when training on one dataset (the left-most column) and testing on another (the upper-most row).
The bold black numbers indicate \emph{in-domain} performance, while the plain black indicate  \emph{cross-domain} performance. 
``Self'' refers to training and testing on the same dataset, same as the bold black numbers and ``mean others'' refers to the average performance over all cross-domain cases.
}
\vspace{-1em}
  \centering
    \begin{tabular}{| c || c | c | c | c || c | c | c |}
    \hline
    Training/Test & VOC2007  & LabelMe & Caltech-101 & SUN09 & Self & Mean others & Percent drop ($\frac{(\mathrm{Self} - \mathrm{MeanOthers}) * 100}{\mathrm{Self}}$)\\
    \hline
    VOC2007 & $\mathbf{72.46}$ & $52.45$ & $89.17$ & $60.00$ & $\mathbf{72.46}$ & $67.20$ & {\color{red}$\sim7\%$}\\
    LabelMe & $54.99$ & $\mathbf{63.74}$ & $79.72$ & $46.90$ & $\mathbf{63.74}$ & $60.54$ & {\color{red}$\sim5\%$ }\\ 
    Caltech-101 & $53.70$ & $44.79$ & $\mathbf{99.53}$ & $44.87$ & $\mathbf{99.53}$ & $47.49$ & {\color{red}$\sim52\%$}\\
    SUN09 & $51.63$ & $50.69$ & $50.71$ & $\mathbf{68.12}$ & $\mathbf{68.12}$ & $51.01$ & {\color{red}$\sim25\%$}\\
    \hline
    Mean others & $53.44$ & $49.31$ & $73.19$ & $50.59$ & $\mathbf{75.96}$ & $56.63$ & {\color{red}$\sim25\%$}\\
    \hline
    \end{tabular}
    \label{tab:vlcs_bias}
\end{table*}
\begin{table*}[!htb]
	\caption{Domain generalization performance accuracy ($\%$) on the VLCS dataset with $\textnormal{DeCAF}_6$ features as inputs. 
	The accuracy of all feature learning-based algorithms: Raw, KPCA, uSCA, DICA, SCA is according to 1-nearest neighbor (1NN) classifier.
	Bold red and bold black indicate the best and the second best performance, respectively.}
	\vspace{-1em}
	\centering
	\begin{tabular}{| c | c || c | c || c | c | c | c | c | c | c | }
	\hline
	Source & Target & 1NN & L-SVM & KPCA & Undo-Bias & UML & LRE-SVM & uSCA & DICA & SCA \\
	\hline
	L,C,S & V & $57.26$ & $58.44$ & $60.22$ & $54.29$ & $56.26$ & $\mathbf{60.58}$ & $58.54$ & $59.62$& {\color{red} $\mathbf{64.36}$} \\	
	V,C,S & L & $52.45$ & $55.21$ & $51.94$ & $58.09$ & $58.50$ & {\color{red}$\mathbf{59.74}$} & $54.08$ & $51.82$& $\mathbf{59.60}$ \\
	V,L,S & C & $\mathbf{90.57}$ & $85.14$ & $90.09$ & $87.50$ & {\color{red} $\mathbf{91.13}$} & $88.11$ & $85.14$ & $78.30$ & $88.92$ \\
	V,L,C & S & $56.95$ & $55.23$ & $55.03$ & $54.21$ & $\mathbf{58.49}$ & $54.88$ & $55.63$ & $55.33$ & {\color{red} $\mathbf{59.29}$}\\
	C,S & V,L & $55.08$ & $55.58$ & $55.64$ & $\mathbf{59.28}$ & $56.47$ & $55.04$ & $53.98$ & $50.90$ & {\color{red} $\mathbf{59.50}$}\\
	C,L & V,S & $52.60$ & $51.80$ & $50.70$ & $\mathbf{55.80}$& $54.72$ & $52.87$ & $49.05$& $55.47$ & {\color{red}$\mathbf{55.96}$}\\
	V,C & L,S & $56.62$ & $59.99$ & $54.66$ & {\color{red}$\mathbf{62.35}$} & $55.49$ & $58.84$ & $55.89$ & $58.08$ &  $\mathbf{60.77}$\\
	\hline
	 \multicolumn{2}{| c ||}{Avg.}&  $60.22$ & $60.20$ & $59.47$ & $\mathbf{61.65}$ & $61.58$ & $61.44$ & $58.90$ & $58.50$ & {\color{red}$\mathbf{64.06}$}\\

	\hline
	\end{tabular}
	\label{tab:exp2_vlcs}
\end{table*}
\begin{table*}[!htb]
	\caption{Domain generalization performance accuracy ($\%$) on the Office+Caltech dataset with $\textnormal{DeCAF}_6$ features as inputs.}
	\vspace{-1em}
	\centering
	\begin{tabular}{| c | c || c | c ||  c | c | c | c | c | c | c |}
	\hline
	Source & Target & 1NN & L-SVM & KPCA & Undo-Bias & UML & LRE-SVM & uSCA &  DICA & SCA\\
	\hline
	W, D, C & A    & $85.39$  & $91.34$ & $89.14$ & $90.98$ & $91.02$ & $\mathbf{91.87}$ & $89.46$ & $90.40$ &  {\color{red}$\mathbf{92.38}$} \\
	A,W,D   & C    & $73.73$  & $84.95$ & $75.87$ &  $85.95$ & $84.59$ & $\mathbf{86.38}$ & $77.15$ & $84.33$ &  {\color{red}$\mathbf{86.73}$} \\
	A,C      & D,W &  $67.92$ & $81.86$ & $78.99$ & $80.49$ & $82.29$ & $\mathbf{84.59}$ & $78.10$ & $79.65$ & {\color{red}$\mathbf{85.84}$}\\
	D,W     & A,C  &  $67.09$ & $77.94$ & $68.84$ & $69.98$ & $\mathbf{79.54}$ & {\color{red}$\mathbf{81.17}$} & $71.74$ &  $69.73$ & $75.54$\\
	\hline
	\multicolumn{2}{ | c ||}{Avg.} & $72.28$ & $84.02$ & $77.71$ & $81.85$ & $84.36$ & {\color{red} $\mathbf{86.00}$} & $79.11$ & $81.02$ & $\mathbf{85.12}$ \\
	\hline
	\end{tabular}
	\label{tab:office_results}
\end{table*}
\begin{table*}[!htb]
	\caption{Domain generalization performance accuracy ($\%$) on the IXMAS dataset with dense trajectory-based features.}
	\vspace{-1em}
	\centering
	\begin{tabular}{| c | c || c | c || c | c | c | c | c | c | c | }
	\hline
	Source & Target & 1NN & L-SVM & KPCA & Undo-Bias & UML &  LRE-SVM & uSCA & DICA & SCA \\
	\hline
	Cam 0,1 & Cam 2,3,4 & $58.24$ & $73.26$ &  $67.77$ & $69.03$ & $74.14$ & $\mathbf{79.96}$ & $66.67$ & $65.93$ & {\color{red}$\mathbf{80.59}$}\\
	Cam 2,3,4 & Cam 0,1 & $20.33$ & $84.07$ & $41.21$ & $60.56$ & $63.79$ & $\mathbf{80.15}$ & $51.09$ & $78.02$ & {\color{red} $\mathbf{85.16}$}\\
	Cam 0,1,2,3 & Cam 4 & $39.56$ & $67.03$ & $59.34$ & $56.84$ & $60.37$ & {\color{red} $\mathbf{74.97}$} & $61.54$ & $62.64$  & $\mathbf{70.33}$\\
	\hline
	\multicolumn{2}{| c ||}{Avg.} & $39.38$ & $72.59$ & $56.10$ & $62.14$ & $66.10$ &  $ \mathbf{78.36}$ & $59.77$& $68.86$ & {\color{red}$\mathbf{78.69}$} \\
	\hline
	\end{tabular}
	\label{tab:ixmas_results}
\end{table*}

\subsection{Baselines and Protocol}
\label{sec:exp2_protocol}
We compared our algorithms with the following baselines:

\begin{itemize}[leftmargin=*]
 \item \textbf{1NN}: 1-nearest neighbor classifier.
 \item \textbf{L-SVM}: SVM classifier with linear kernel.
 \item \textbf{KPCA}~\cite{Scholkopf1998}: Kernel Principal Component Analysis.
 \item \textbf{Undo-Bias}~\cite{Khosla2012}: a multi-task SVM-based algorithm for undoing dataset bias. Three hyper-parameters ($\lambda, C_1, C_2$) require tuning. 
 Since the original formulation was designed for binary classification, we performed the following setup for multi-class classification purposes.
 We trained $C$ individual Undo-Bias classifiers $f^{ub}_k: \bbR^d \rightarrow \{-1, 1\}, \forall k = 1, \ldots, C$, where $C$ is the number of classes.
 At the prediction stage, given a test instance $(\hat{\bx}, \hat{y} )$ we computed $\hat{Y} := \{ k | \forall k=1, \ldots, C: f^{ub}_k = 1\}$.
 Finally, we verified whether $\hat{y} \in \hat{Y}$.
 \item \textbf{UML}~\cite{Fang2013}: a structural metric learning-based algorithm that aims to learn a less biased distance metric for classification tasks.
 The initial tuning proposal for this method was using a set of weakly-labeled data retrieved from querying class labels to search engine. However, here we tuned the hyper-parameters using the same k-fold cross-validation strategy as others for a fair comparison.
 \item \textbf{DICA}~\cite{Muandet2013}: a kernel feature extraction method for domain generalization. DICA has three tunable hyper-parameters.
 \item \textbf{LRE-SVM}~\cite{Xu2014}: a non-linear exemplar-SVMs model with a nuclear norm regularization to impose a low-rank \emph{likelihood matrix}. 
 LRE-SVM has four hyper-parameters ($\lambda_1$, $\lambda_2$, $C_1$, $C_2$) that require tuning.
\end{itemize}
Undo-Bias, UML, and LRE-SVM are the prior state-of-the-art domain generalization algorithms for object recognition tasks.
Note that Undo-Bias, DICA, and UML cannot be applied to the domain adaptation setting -- they do not all

We used 1-nearest neighbor (1NN) as the base classifier for all feature learning-based algorithms: KPCA, DICA, uSCA/uDICA, and SCA.
The tunable hyper-parameters were selected according to labels from source domains.
For all kernel-based methods, the kernel function is the RBF kernel, $k(\ba, \bb) = \exp( - \frac{\| \ba - \bb\|^2}{\sigma^2}) $, with a kernel bandwidth $\sigma$ computed by \emph{median heuristic}.
Note that the unsupervised DICA (uDICA) is almost identical to uSCA in this case. The only difference is that uSCA has a control parameter $\delta > 0$ for the domain scatter/distributional variance term.

\vspace{-1em}
\subsection{Results on the VLCS Dataset}
On this dataset, we first conducted the standard training-test evaluation using 1-nearest neighbor (1NN), i.e., learning the model on a training set from one domain and testing it on a test set from another domain, to check the groundtruth performance and also to identify the existence of the dataset bias.
The groundtruth evaluation results are summarized in Table~\ref{tab:vlcs_bias}.
In general, the dataset bias indeed exists despite the use of the state-of-the-art deep convolution neural network features $\textnormal{DeCAF}_6$.
For example, the average cross-domain performance, i.e., "Mean others", is $56.63\%$, which is $25\%$ drop from the corresponding in-domain performance: $75.96\%$.
In particular, Caltech-101 has the highest bias, while LabelMe is the least biased dataset indicated by the largest and smallest performance drop, respectively.

We then evaluated the domain generalization performance over seven cross-domain recognition tasks.
The complete results are summarized in Table~\ref{tab:exp2_vlcs}.
We can see that SCA is the best model on 5 out of 7 tasks, outperforms the prior state-of-the-art, LRE-SVM.
It almost always has better performance than the `raw' baseline, except when Caltech-101 is the target domain.
On average, SCA is about $2\%$ better than its closest competitor on this dataset, Undo-Bias.
The VLCS cross-domain recognition is a hard task in general, since the best model (SCA) only provides $<4\%$ average improvement over the raw baseline.
Furthermore, three algorithms, two of which are the domain generalization-based methods (uSCA, DICA), cannot achieve even better performance than the raw baseline.

\subsection{Results on the Office+Caltech Dataset}
We evaluated our algorithms on several cross-domain cases constructed from the Office+Caltech dataset. 
The detailed evaluation results on four cases with $\textnormal{DeCAF}_6$ are reported in Table~\ref{tab:office_results}.
We do not report other cross-domain cases that are possibly constructed from this dataset, such as $A,D,C \rightarrow W$ and $A,W,C\rightarrow D$, 
since 1NN on Raw features has already provided high accuracies ($> 95\%$).

The closest competitor to SCA is LRE-SVM.
Although LRE-SVM performs best on average, SCA has the best performance on three out of four cross-domain cases and comes second on average.
The only case when SCA underperforms LRE-SVM is that of $D,W \rightarrow A,C$.
Note that the LRE-SVM algorithm is more complex than SCA both in the optimization procedure and in the number of tunable hyper-parameters.

However, the unsupervised version of our algorithm, uSCA, which is the same as uDICA~\cite{Muandet2013} in the domain generalization case, cannot compete with the state-of-the-art models. 
It is only slightly better than KPCA on average.
This suggests that incorporating labeled information from source domains during feature learning does improve domain generalization on the Office+Caltech cases.

\subsection{Results on the IXMAS dataset}
Table~\ref{tab:ixmas_results} summarizes the classification accuracies on the IXMAS dataset over three cross-domain cases.
We can see that the standard baselines (Raw, KPCA) cannot match other algorithms with domain generalization strategies.
In this dataset, SCA has the best performance on two out of three cases and on average. 
In particular, SCA is significantly better than others on Cam 2,3,4 $\rightarrow$ Cam 0,1 case.
LRE-SVM remains the closest competitor of SCA -- it has the second best average performance with one best cross-domain case.

\vspace{-1em}
\subsection{Runtime Performance}
Next we report the average (training) runtime performance over all cross-domain recognition tasks in each dataset.
All algorithms were executed using the same software and machine as described in Section~\ref{sec:exp1_runtime}.
From Table~\ref{tab:runtime_dg}, we can see that the runtime of SCA is on par with KPCA and DICA, which is expected since they utilize the same optimization procedure: a single run with a generalized eigenvalue decomposition.

SCA is significantly faster than some prior state-of-the-art domain generalization methods (Undo-Bias, UML, and LRE-SVM).
For example, on the VLCS dataset, Undo-Bias, UML, and LRE-SVM require $\sim30$ minutes, while SCA only needs $\sim5$ minutes average training time.
An analogous trend can also be seen in the case of Office+Caltech and IXMAS datasets.
This outcome indicates that SCA is better suited for domain generalization tasks than the competing algorithms if a training stage in real time is required.

\begin{table}[!htb]
	\caption{Average domain generalization runtime (seconds) over all cross-domain recognition tasks in each dataset.}
	\vspace{-1em}
	\centering
	\scalebox{0.82}{
	\begin{tabular}{| c || c | c | c | c | c | c |}
	\hline
	Dataset & KPCA & Undo-Bias & DICA & UML & LRE-SVM & SCA \\
	\hline
	VLCS 		& $201.99$ & $1,925.54$ & $336.92$ & $1,652.67$ & $2,161.60$ & $300.94$\\
	Office+Caltech 	& $6.53$ &  $589.50$ &  $17.90$ & $413.25$ & $695.31$  & $18.49$ \\
	IXMAS 		& $0.50$ & $49.14$ &  $0.79$ & $57.67$ & $65.47$ & $0.96$ \\
	\hline
	\end{tabular}
	}
	\label{tab:runtime_dg}
\end{table}

\vspace{-1em}
\section{Conclusions}
\label{sec:conc}
The scatter-based objective function is 
a straightforward way to encode the relevant structure of the domain adaptation and domain generalization problems. 
SCA uses variances between subsets of the data to construct a linear transformation that dampens unimportant distinctions (within labels and between domains) and amplifies useful distinctions (between labels and overall variability). 
Extensive experiments on several cross-domain image datasets show that SCA is much faster than competing algorithms and provides the state-of-the-art performance on both domain adaptation and domain generalization.

Our theoretical analysis shows that scatter with two input domains, i.e., domain scatter, provides generalization bounds for domain adaptation~\cite{Mansour2009}. In the setting of domain generalization, recall remark~\ref{rem:gen_scatter}, prior work has shown that the distributional variance (which is a special case of scatter) arises as one of the terms controlling generalization performance \cite{Muandet2013}. Scatter is thus a unifying quantity that controls generalization performance in domain adaptation and generalization.

SCA is a natural extension of Kernel PCA, Kernel Fisher Discriminant and TCA. 
In contrast, many domain adaptation methods use objective functions that combine the total variance and MMD with quantities that are fundamentally different in kind such as the graph Laplacian \cite{Long:2013aa}, sparsity constraints \cite{Long:2013aa,Long2014a}, the Hilbert-Schmidt independence criterion \cite{Pan2011} or the central subspace \cite{Muandet2013}. 
SCA can easily be extended to semi-supervised domain adaptation, by incorporating target labels into the class scatters. 
Finally, we remark that it should be possible to speed up SCA for large-scale problems using random features~\cite{Rahimi:07}. 

%

In general, dataset bias remains far from solved. 
Existing algorithms perform satisfactorily ($\geq80\%$ accuracy) only in several cross-domain tasks, even when using powerful feature extraction methods such as $\textnormal{DeCAF}_6$ (for images) and dense trajectory-based features (for videos).
Using less powerful features (raw pixels or SURF-BoW) is clearly unsatisfactory.
Thus, it is crucial to develop more fundamental feature learning algorithms that significantly reduce dataset bias in a wide range of situations.

\vspace{-1em}
\section{Acknowledgments} 
The authors would like to thank Zheng Xu for sharing the extracted dense trajectories features from the IXMAS dataset, and also Chen Fang for sharing the \emph{Unbiased Metric Learning} code and useful discussions.
\vspace{-1em}

\bibliographystyle{IEEEtran}
{
\footnotesize
\bibliography{pamirefs}
}

\begin{IEEEbiography}[{\includegraphics[width=1in,height=1.25in,clip,keepaspectratio]{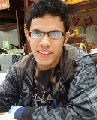}}]{Muhammad Ghifary} 
received BEng and MEng degrees from Institut Teknologi Bandung (ITB), Indonesia, 
and has been awarded a PhD degree from Victoria University of Wellington, New Zealand.
He is currently working at Weta Digital as a research intern.
Previously, he worked at ITB and Catholic Parahyangan University as a lecturer.
He is a student member of the IEEE and AAAI. 
He has published his work at conferences such as ICASSP, ICCV, and ICML.
His main research interests include domain adaptation, transfer learning, representation learning, deep learning, and their applications in computer vision.
\end{IEEEbiography}
\vspace{-3em}
\begin{IEEEbiography}[{\includegraphics[width=1in,height=1.25in,clip,keepaspectratio]{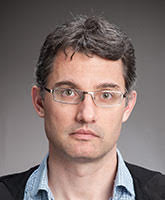}}]{David Balduzzi}
received a PhD degree in mathematics from the University of Chicago in 2006. He subsequently held positions at the Max Planck Institute for Intelligent Systems and ETH Zurich. He is currently a senior lecturer in the School of Mathematics and Statistics at Victoria University Wellington. His research interests are machine learning and computational neuroscience. He has published his work at conferences such as NIPS, ICML, UAI, AAAI, AAMAS and ICCV, and in journals such as Annals of Statistics, Network Science and PLoS Computational Biology.
\end{IEEEbiography}
\vspace{-3em}
\begin{IEEEbiography}[{\includegraphics[width=1in,height=1.25in,clip,keepaspectratio]{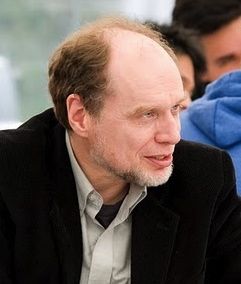}}]{W. Bastiaan Kleijn}
has been a Professor at Victoria University of Wellington since 2010. 
He is also a Professor at Delft University of Technology (part-time) and was a Professor at KTH, where he headed the Sound and Image Processing Laboratory until he moved to New Zealand. Before joining KTH in 1996, he worked at AT \&T Bell Laboratories (Research) on speech processing. 
He was a founder of Global IP Solutions, which developed voice and video processing engines for, among others, Google, Skype, and Yahoo and was sold to Google in 2010. 
He holds a Ph.D. in electrical engineering from Delft University of Technology and an M.S.E.E. from Stanford. 
He also earned a Ph.D. in soil science and an M.S. in physic from the University of California, Riverside. 
He is a Fellow of the IEEE.
\end{IEEEbiography}
\vspace{-3em}
\begin{IEEEbiography}[{\includegraphics[width=1in,height=1.25in,clip,keepaspectratio]{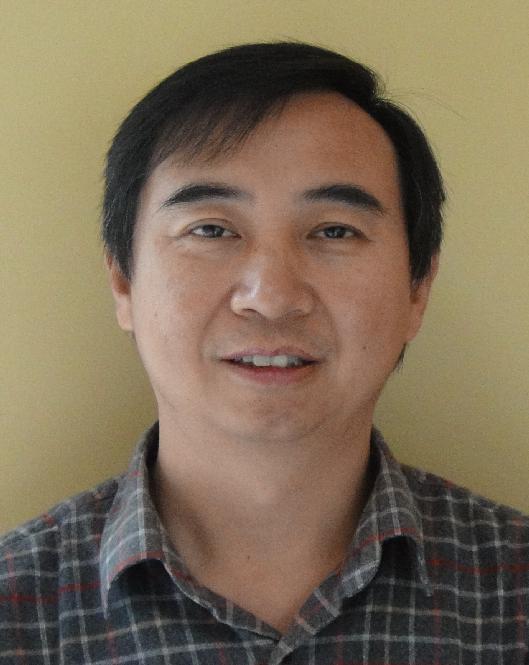}}]{Mengjie Zhang}
received the B.E. and M.E. degrees from Artificial Intelligence Research Center, Agricultural University of Hebei, Hebei, China, and the Ph.D. degree in computer science from RMIT
University, Melbourne, VIC, Australia, in 1989, 1992, and 2000, respectively.
Since 2000, he has been with the Victoria University of Wellington, Wellington, New Zealand, where he is currently Professor of Computer Science, Head of the Evolutionary Computation Research Group, and the Associate Dean (Research and Innovation) in the Faculty of Engineering.
His current research interests include evolutionary computation, particularly genetic programming, particle swarm optimization, and learning classifier systems with application areas of image analysis, multiobjective optimization, classification with unbalanced data, feature selection and reduction, and job shop scheduling. 
He has published over 350 academic papers in refereed international journals
and conferences.

\end{IEEEbiography} 
\end{document}